\documentclass{article}
\usepackage{microtype}
\usepackage{graphicx}
\usepackage{subfigure}
\usepackage{booktabs}
\usepackage{hyperref}

\usepackage[accepted]{icml2024}
\usepackage{amsmath}
\usepackage{amssymb}
\usepackage{mathtools}
\usepackage{amsthm}
\usepackage[capitalize,noabbrev]{cleveref}

\usepackage{tabulary,multirow}
\usepackage{wrapfig}
\usepackage{amssymb}
\usepackage{pifont}
\newcommand{\cmark}{\ding{51}}%
\newcommand{\xmark}{\ding{55}}%
\usepackage{enumitem}

\usepackage[frozencache,cachedir=.]{minted}
\definecolor{LightGray}{gray}{0.9}

\theoremstyle{plain}
\newtheorem{theorem}{Theorem}[section]
\newtheorem{proposition}[theorem]{Proposition}
\newtheorem{lemma}[theorem]{Lemma}

\theoremstyle{definition}
\newtheorem{definition}[theorem]{Definition}

\newtheorem{property}[theorem]{Property}
\theoremstyle{remark}
\newtheorem{remark}[theorem]{Remark}
\DeclareMathOperator*{\argmax}{arg\,max}

\usepackage[textsize=tiny]{todonotes}

\usepackage[acronym]{glossaries}
\makeglossaries
\newacronym{IID}{IID}{Independent-identically-distributed}
\newacronym{OOD}{OOD}{Out-of-Distribution}
\newacronym{DNN}{DNN}{Deep Neural Network}
\newacronym{ERM}{ERM}{Empirical Risk Minimization}
\newacronym{MLE}{MLE}{Maximum Likelihood Estimation}
\newacronym{BNN}{BNN}{Bayesian Neural Network}
\newacronym{SNGP}{SNGP}{Spectral-normalized Neural Gaussian Process}
\newacronym{MC}{MC}{Monte-Carlo}
\newacronym{MFVI}{MFVI}{Mean-Field Variational Inference}
\newacronym{NLL}{NLL}{Negative log-likelihood}
\newacronym{ECE}{ECE}{Expected Calibration Error}
\newacronym{SOTA}{SOTA}{State-of-the-art}
\newacronym{MIMO}{MIMO}{Multi-input Multi-output}
\newacronym{DUQ}{DUQ}{Deterministic Uncertainty Quantification}
\newacronym{DDU}{DDU}{Deep Deterministic Uncertainty}
\newacronym{DUE}{DUE}{Deterministic Uncertainty Estimation}
\newacronym{NatPN}{NatPN}{Natural Posterior Network}
\newacronym{AUPR}{AUPR}{Area Under the Precision-Recall}
\newacronym{AUROC}{AUROC}{Area Under the Receiver Operating Characteristic}

\icmltitlerunning{Density-Softmax: Efficient Test-time Model for Uncertainty Estimation and Robustness under Distribution Shifts}

\begin{document}
\twocolumn[
\icmltitle{Density-Softmax: Efficient Test-time Model for Uncertainty Estimation and Robustness under Distribution Shifts
}
\begin{icmlauthorlist}
\icmlauthor{Ha Manh Bui}{jhu}
\icmlauthor{Anqi Liu}{jhu}
\end{icmlauthorlist}
\icmlaffiliation{jhu}{Department of Computer Science, Johns Hopkins University, Baltimore, MD, U.S.A}
\icmlcorrespondingauthor{Ha Manh Bui}{hb.buimanhha@gmail.com}
\icmlkeywords{Machine Learning, ICML}
\vskip 0.3in
]
\printAffiliationsAndNotice{}

\begin{abstract}
Sampling-based methods, e.g., Deep Ensembles and Bayesian Neural Nets have become promising approaches to improve the quality of uncertainty estimation and robust generalization. However, they suffer from a large model size and high latency at test-time, which limits the scalability needed for low-resource devices and real-time applications. To resolve these computational issues, we propose Density-Softmax, a sampling-free deterministic framework via combining a density function built on a Lipschitz-constrained feature extractor with the softmax layer. Theoretically, we show that our model is the solution of minimax uncertainty risk and is distance-aware on feature space, thus reducing the over-confidence of the standard softmax under distribution shifts. Empirically, our method enjoys competitive results with state-of-the-art techniques in terms of uncertainty and robustness, while having a lower number of model parameters and a lower latency at test-time. 
\end{abstract}
\section{Introduction}
The ability of models to produce high-quality uncertainty estimation and robustness is crucial for reliable \acrfull{DNN} in high-stake applications (e.g., healthcare, finance, decision-making, etc.). In principle, a \textbf{reliable model} permits graceful failure, signaling when it is likely to be wrong (\textbf{uncertainty}), and also generalizes better under distribution shifts (\textbf{robustness})~\citep{tran2022plex}. Additionally, to be widely used in real-world scenarios, the reliable \acrshort{DNN} model also necessarily needs to be fast and lightweight (\textbf{efficiency}). This efficiency criterion can be considered in two phases, training-time and test-time. At \textbf{training-time}, an inefficient \acrshort{DNN} model might be acceptable given the high computational resources in development. However, at \textbf{test-time}, inefficiency is a critical issue for users when the model needs to be deployed on low-resource devices and in real-time applications~\citep{hinton2015distilling}. 

\begin{table}[t!]
    \centering
    \scalebox{0.76}{
    \begin{tabular}{ccccc}\\
    \toprule  
    Method & $\begin{matrix}
    \text{Uncertainty} \\
    \text{quality}
    \end{matrix}$ & $\begin{matrix}
    \text{Robustness} \\
    \text{quality}
    \end{matrix}$ & $\begin{matrix}
    \text{Test-time} \\
    \text{efficiency}
    \end{matrix}$ & $\begin{matrix}
    \text{Without prior} \\
    \text{requirement}
    \end{matrix}$\\\midrule
    Deterministic & \xmark & \xmark & \cmark & \cmark\\
    Bayesian  & \cmark & \xmark & \cmark & \xmark\\
    Ensembles & \cmark & \cmark & \xmark & \cmark\\
    Ours & \cmark & \cmark & \cmark & \cmark\\
    \bottomrule
    \end{tabular}}
    \caption{A comparison between methods regarding uncertainty, robustness quality, test-time efficiency (lightweight \& fast), and whether pre-defined prior hyper-parameters are required.}
    \vspace{-0.1in}
    \label{tab:teaser}
\end{table}

Deterministic~\acrfull{ERM}~\citep{vapnik1998erm} model nowadays can be efficient due to being \textbf{sampling-free} with $\mathcal{O}(1)$ sample complexity, i.e., it only needs a single forward pass with a single DNN model to produce the softmax probability. However, it often struggles with over-confidence and over-fitting. This poor performance usually occurs when the test data is far and does not come from the same distribution as the training set~\citep{minderer2021revisiting,bui2021exploiting,bui2022benchmark,ovadia2019can,guo2017on}.

To improve uncertainty estimation and robustness under distribution shifts, recent \textbf{sampling-based} approaches with $\mathcal{O}(M)$ sample complexity, where $M$ is the number of sampling times, have shown promising results~\citep{collier2021correlated,dusenberry2020rank1,Wen2020BatchEnsemble:}. Among these works, the best performance in practice so far is based on Deep Ensembles~\citep{lakshminarayanan2017ensemble,tran2022plex, nado2021uncertainty}. However, this approach suffers from a heavy computational burden as it requires more model parameters and multiple forward passes, leading to inefficiency at test-time. To tackle this challenge, sampling-free methods~\citep{mukhoti2022deep,charpentier2022natural,Liu2020SNGP,havasi2021training}, and lightweight sampling-based models~\citep{Wen2020BatchEnsemble:,dusenberry2020rank1} have been recently proposed. Nevertheless, besides generally performing worse than Deep Ensembles, these methods are also less computationally efficient than Deterministic~\acrshort{ERM}~\citep{nado2021uncertainty} (e.g., Tab.~\ref{tab:teaser}).

\begin{figure*}[t!]
\begin{center}
  \includegraphics[width=1.0\linewidth]{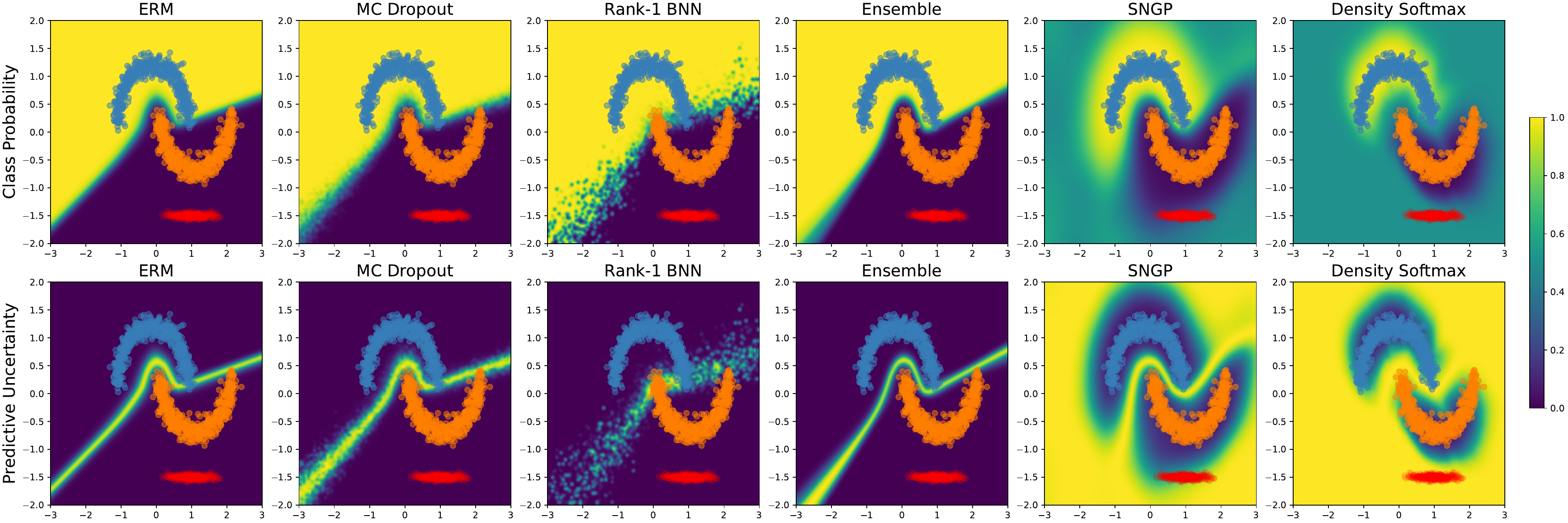}
\end{center}
  \caption{The class probability $p(y|x)$ (Top Row) and predictive uncertainty $var(y|x) = p(y|x) \cdot (1 - p(y|x))$ surface (Bottom Row) as background colors in a comparison between our Density-Softmax and different approaches on the two moons 2D classification. Training data for positive (\textcolor{orange}{Orange}) and negative classes (\textcolor{blue}{Blue}) are shown. \acrshort{OOD} data (\textcolor{red}{Red}) are not observed during training. \textbf{Our Density-Softmax achieves distance awareness with a uniform class probability and high uncertainty value on \textcolor{red}{OOD data}}. A quick demo is available at \href{https://colab.research.google.com/drive/1dqaacHzOHUPFhBDcUw7yGL-zv7GSG-8P?usp=sharing}{https://colab.research.google.com/drive/1dqaacHzOHUPFhBDcUw7yGL-zv7GSG-8P?usp=sharing}.}
\label{fig:2dmoon}
\end{figure*}

Toward a model that keeps the Deep Ensembles performance with test-time efficiency similar to Deterministic~\acrshort{ERM}, we introduce Density-Softmax, a sampling-free single-model via a combination of a density function and a Lipschitz-constrained feature extractor. By using regularization to enforce the 1-Lipschitz constraint in training, our model can improve the robustness under distribution shifts. In addition, by combining the feature density function with the softmax layer via a single forward pass, our method only needs a small number of additional parameters and latency when compared to Deterministic~\acrshort{ERM}. Importantly, this combination helps Density-Softmax be feature distance aware, i.e., its associated uncertainty metrics are monotonic functions of feature distance metrics, leading to a good uncertainty notion of when the test feature is near or far from the training set. This distance-aware property is important to help \acrshort{DNN} for both calibration and \acrshort{OOD} detection, however, it is often not guaranteed for typical \acrshort{DNN} models~\citep{Liu2020SNGP,bui2024density} (e.g., Fig.~\ref{fig:2dmoon}).

To summarize, our model includes three main components: a Lipschitz-constrained feature extractor, a lightweight Normalizing-Flows density model on the feature space, and a classifier with the softmax layer. In training-time, the feature extractor is pre-trained using \acrshort{ERM} objective and gradient-penalty regularization, aiming to achieve the 1-Lipschitz constraint. After that, the Normalizing-Flows model estimates a marginal density of the learned low-dimensional feature space. Finally, the classifier is fine-tuned with feature likelihood from the density model. In test-time, the feature likelihood is combined directly with the logit vector of the classifier to produce a softmax probability with only a single forward pass.

Our contributions can be summarized as follows:
\begin{enumerate}
    \item We introduce Density-Softmax, a reliable, sampling-free, and single \acrshort{DNN} framework via a direct combination of a density function built on a Lipschitz-constrained feature extractor with the softmax layer. Notably, our algorithm does not require any data augmentation~\citep{hendrycks*2020augmix,zhang2018mixup}, the \acrshort{OOD} set, or the post-hoc re-calibration technique~\citep{mukhoti2022deep} in training.
    \item We formally prove that our model is the solution to the minimax uncertainty risk, distance awareness on the feature space, and can reduce over-confidence of the standard softmax when the test feature is far from the training set. 
    \item We empirically show that our framework achieves a competitive robust generalization and uncertainty estimation performance with \acrshort{SOTA} on the Toy dataset with ResFFN-12-128, shifted CIFAR-10-100 with Wide~Resnet-28-10, and ImageNet with ResNet-50. Importantly, Density-Softmax requires only a single forward pass and a lightweight feature density function. Therefore it has fewer parameters and is much faster than other baselines at test-time.
\end{enumerate}

\section{Related work}
\textbf{Uncertainty and Robustness.} \citet{nado2021uncertainty,bui2022benchmark} has studied the uncertainty and robustness of modern \acrshort{DNN} approaches extensively in the benchmark of \acrshort{SOTA} baselines, mainly evaluating \acrshort{NLL}, in-out accuracy for robust generalization, and \acrshort{ECE} for calibrated uncertainty. More modern discussion of reliable \acrshort{DNN} can be found in the literature of~\citet{tran2022plex}. Among these methods, sampling-based approaches are widely used, from Gaussian Process~\citep{gardner2018GPyTorch,lee2018deepGP}, Dropout~\citep{gal2016dropouttheory,gal2017concretedropout}, \acrshort{BNN}~\citep{blundell2015weight,wen2018flipout,maddox2019swag}, to the \acrshort{SOTA} Ensembles~\citep{lakshminarayanan2017ensemble}. However, these methods often have a high number of model parameters. To resolve this issue, lightweight sampling-based models, e.g., BatchEnsemble~\citep{Wen2020BatchEnsemble:}, Rank-1~BNN~\citep{dusenberry2020rank1}, and Heteroscedastic~\citep{collier2021correlated} have been proposed recently.

\textbf{Sampling-free methods.} To tackle the scalability challenge in sampling-based methods, novel sampling-free methods have been investigated, including replacing the loss function~\citep{wei2022logitnorm,malinin2018priornet,malinin2019reverseKL,kotelevskii2022NUQ,karandikar2021soft}, the output layer~\citep{bui2024density,van2020uncertainty,mukhoti2022deep,tagasovska2019single,wang2022Vim,liu2020energybase}, or computing a closed-form posterior in Bayesian inference~\citep{kopetzki2021evaluating,sensoy2018evidential,riquelme2018deep,snoek2015scalableBNN,kristiadi2020being,charpentier2022natural}. Nevertheless, these methods often only focus on improving uncertainty quality without improving accuracy or even using additional re-calibration set to enhance this performance~\citep{mukhoti2022deep}. In the scope of uncertainty and robustness~\citep{nado2021uncertainty}, there exist distillation approaches~\citep{vadera20Generalized,Malinin2020Ensemble} and closest to our work is \acrshort{MIMO}~\citep{havasi2021training}, Posterior~Net~\citep{charpentier2020postnet}, and \acrshort{SNGP}~\citep{Liu2020SNGP}. Although these methods can improve efficiency at test-time, they still underperform Deep Ensembles regarding reliability and Deterministic~\acrshort{ERM} regarding the model's test-time efficiency.

\textbf{Improving uncertainty quality via density estimation.} Enhancing uncertainty estimates via density function has shown promising results in practice~\citep{bui2024density,kuleshov2022sharpness,charpentier2020postnet,mukhoti2022deep,kotelevskii2022NUQ,dosovitskiy2021an,lee2018asimple,sun2022OOD}. That said, \citet{dosovitskiy2021an,lee2018asimple,sun2022OOD} suffer from a heavy computational burden at test-time and only focus on \acrshort{OOD} detection. Meanwhile, our efficient method additionally provides both empirical and theoretical analysis of the calibration level (details are in Apd.~\ref{apd:additional_discuss}). Theoretically, \citet{charpentier2022natural} proves with a small number of parameters, the Normalizing-Flows model can improve the uncertainty notion of \acrshort{DNN} by providing a high likelihood when the test feature is close and a low likelihood when that is far from training data. However, this work customizes the last layer of \acrshort{DNN} with sensitive priors, not normalized by a natural exponent function, resulting in a bad robustness performance~\citep{nado2021uncertainty}. In our work, Density-Softmax utilizes a Normalizing-Flows model on the Lipschitz-constrained feature together with the logit vector of the classifier to improve the robust accuracy, uncertainty quality, and test-time efficiency.

\textbf{Improving robustness via 1-Lipschitz constraint.} 1-Lipschitz Neural Nets have been widely used to train certifiably robust classifiers in practice~\citep{tsuzuku2018lipschitz,bthune2022pay,li2019preventing,searcod_metric_2006}. It is theoretically shown to be able to defend against adversarial attack~\citep{li2019preventing}, preserve accuracy on \acrshort{IID}, and improve the robust generalization on \acrshort{OOD} data~\citep{bthune2022pay}. However, it is not clear whether this property can contribute to better uncertainty estimation. In this work, we integrate the Lipschitz-constrained feature extractor from gradient-penalty regularization~\citep{gulrajani2017improved} with our density estimation component to achieve better robust accuracy, uncertainty quality, and test-time efficiency.
\section{Density-Softmax}
\subsection{Notation and Problem setting}
Let $\mathcal{X}$ and $\mathcal{Y}$ be the sample and label space. Denote the set of joint probability distributions on $\mathcal{X} \times \mathcal{Y}$ by $\mathcal{P}_{\mathcal{X} \times \mathcal{Y}}$. A dataset is defined by a joint distribution $\mathbb{P}(X,Y) \in \mathcal{P}_{\mathcal{X} \times \mathcal{Y}}$, and let $\mathcal{P}$ be a measure on $\mathcal{P}_{\mathcal{X} \times \mathcal{Y}}$, i.e., whose realizations are distributions on $\mathcal{X} \times \mathcal{Y}$. Denote the training set by $D_s = \{(x_s^i, y_s^i)\}_{i=1}^{n_s}$, where $n_s$ is the number of data points in $D_s$, s.t., $(x_s, y_s) \sim \mathbb{P}_s(X,Y)$ and $\mathbb{P}_s(X,Y) \sim \mathcal{P}$. In the standard learning setting, a learning model that is only trained on $D_s$, arrives at a good generalization performance on the test set $D_t = \{(x_t^i, y_t^i)\}_{i=1}^{n_t}$, where $n_t$ is the number of data points in $D_t$, s.t., $(x_t, y_t) \sim \mathbb{P}_t(X,Y)$ and $\mathbb{P}_t(X,Y) \sim \mathcal{P}$. In the \acrfull{IID} setting, $\mathbb{P}_t(X,Y)$ is similar to $\mathbb{P}_s(X,Y)$, and let us use $\mathbb{P}_{iid}(X,Y)$ to represent the \acrshort{IID} test distribution. In contrast, $\mathbb{P}_t(X,Y)$ is different with $\mathbb{P}_s(X,Y)$ if $D_t$ is \acrfull{OOD} data, and let us use $\mathbb{P}_{ood}(X,Y)$ to represent the \acrshort{OOD} test distribution.

In the classification setting of representation learning, we predict a target $y \in \mathcal{Y}$, where $\mathcal{Y}$ is discrete with $K$ possible categories by using a forecast $h = \sigma(g \circ f)$, which composites a feature extractor $f: \mathcal{X} \rightarrow \mathcal{Z}$, where $\mathcal{Z}$ is feature space, a classifier embedding $g: \mathcal{Z} \rightarrow \mathbb{R}^{K}$, and a softmax layer $\sigma: \mathbb{R}^{K} \rightarrow \Delta _y$ which outputs a probability distribution $W(y): \mathcal{Y} \rightarrow  [0,1]$ within the set $\Delta_y$ of distributions over $\mathcal{Y}$; the value of probability density function of $W$ is $w$.

\subsection{Method Overview} 
Toward a reliable and efficient test-time framework, we introduce Density-Softmax. Specifically, to improve uncertainty quantification, our idea is based on the solution of the \textbf{minimax uncertainty risk}~\citep{Meinke2020Towards}, i.e.,
\begin{align}\label{eq:minimax}
    \inf_{\mathbb{P}(Y|X) \in \mathcal{P}} \left [ \sup_{\mathbb{P}^*(Y|X) \in \mathcal{P}^*} S(\mathbb{P}(Y|X), \mathbb{P}^*(Y|X)) \right ],
\end{align}
where $S(.,\mathbb{P}^*(Y|X))$ is strictly proper scoring rules, $\mathbb{P}(Y|X)$ is predictive, and $\mathbb{P}^*(Y|X)$ is the data-generation distribution. When $\mathcal{X}_{ood}=\mathcal{X}/\mathcal{X}_{iid}$, for the Brier Score~\citep{brier1950verification}, the solution of the risk was shown by~\citet{Liu2020SNGP} as
\begin{align}
    \mathbb{P}(Y|X) = \mathbb{P}(Y|X_{iid}) \mathbb{P}^*(X_{iid}) + \mathbb{U}(Y|X_{ood}) \mathbb{P}^*(X_{ood}),
\end{align}
where $X_{iid}$ is \acrshort{IID}, $X_{ood}$ is \acrshort{OOD} sample variable, and $\mathbb{U}$ stands for uniform distribution. According to this result, we summarize the overview of Density-Softmax, a solution of the minimax uncertainty risk by Thm.~\ref{thm:optimal}, in Fig.~\ref{fig:framework}.
\begin{figure}[t!]
\begin{center}
\includegraphics[width=1.0\linewidth]{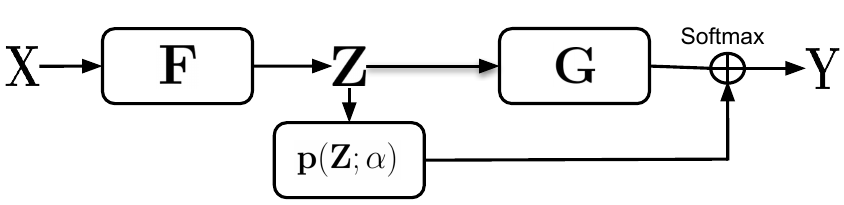}
\end{center}
   \caption{The overall architectures of Density-Softmax, including an encoder $f$, a classifer $g$, and a density function $p(Z;\alpha)$. The rectangle boxes represent these functions. The circle with two cross lines represents the softmax layer. The 3 training steps and inference process follow Algorithm~\ref{alg:algorithm}.}
\label{fig:framework}
\end{figure}
Under our framework, the predictive distribution is equivalent to 
\begin{align}
    \mathbb{P}(Y|X) = \sigma(p(Z;\alpha) \cdot g(Z))\text{, with } Z = f(X),
\end{align}
where $f: \mathcal{X} \rightarrow \mathcal{Z}$ is the feature extractor, $g: \mathcal{Z} \rightarrow \mathbb{R}^K$ is the last classifier layer, $\sigma: \mathbb{R}^K \rightarrow \Delta_y$ is the softmax layer, and $p(Z;\alpha)$ is the density function to measure the distance on feature space $\mathcal{Z}$ with parameter $\alpha$. 

To improve robustness and test-time efficiency, we use the gradient-penalty regularization~\citep{gulrajani2017improved} to enforce the feature extractor $f$ to be 1-Lipschitz based on the Rademacher theorem:
\begin{theorem}\label{theo:radamacher}~\citep{federer2014geometric}
    If $f: \mathcal{X} \rightarrow \mathcal{Z}$ is a locally Lipschitz continuous function, then $f$ is differentiable almost everywhere. Moreover, if $f$ is Lipschitz continuous, then $L(f)=\sup_{x\in \mathbb{R}^n}||\nabla_x f(x)||_2,$ where $L(f)$ is the Lipschitz constant of $f$.
\end{theorem}
\begin{remark}\label{remark:1-Lipschitz} The gradient-penalty $(||\nabla_x f(x)||_2-1)^2$ enforces $\sup_{x\in \mathbb{R}^n}||\nabla_x f(x)||_2 = 1$ and Thm.~\ref{theo:radamacher} suggests $f(x)$ satisfy 1-Lipstchiz constraint by $L(f)=1$, i.e., $||f(x_1)-f(x_2)||_2\leq ||x_1 - x_2||_2$ (details are in Apd.~\ref{apd:additional_discuss}). This prevents $f(x)$ from being overly sensitive to the meaningless perturbations of the sample and assures that if the sample is similar, the feature will also be similar. This 1-Lipstchiz $f(x)$ is proved to be robust on the corrupted data by the Local Robustness Certificates property:
\end{remark}
\begin{property}~\citep{tsuzuku2018lipschitz} 
    For any 1-Lipschitz $f(x)$, i.e., $L(f)=1$, the robustness radius $\epsilon$ of binary classifier $c = sign \circ f$ at example $x$ verifies $\epsilon \geq ||f(x)||$, where $\epsilon = \min_{\delta \in \mathcal{X}}||\delta|| \text{ s.t., } c(x+\delta) \neq c(x)$. (This can be generalized to the multi-class case by~\citet{bthune2022pay}).
\end{property}

\subsection{Algorithm}
\textbf{Training.} Based on the motivation mentioned above, in the first training step, we optimize the model by using \acrshort{ERM}~\citep{vapnik1998erm} and gradient-penalty regularization~\citep{gulrajani2017improved} from training data $D_s$ by solving
\begin{align}\label{method:1st-optimization}
    \min_{\theta_{g,f}} \{ \mathbb{E}_{(x,y) \sim D_s} [ &-y \log \left ( \sigma(g(f(x))) \right )\nonumber\\
    &+ \lambda (||\nabla_x f(x)||_2-1)^2 ] \},
\end{align}
where $\theta_{g,f}$ is the parameter of encoder $f$ and classifier $g$, $\lambda$ is the gradient-penalty coefficient, and $||\nabla_x f(x)||_2$ is the Spectral norm of the Jacobian matrix $\nabla_x f(x)$.

After that, we freeze the parameter $\theta_{f}$ of $f$ to estimate density on the learned representation space $\mathcal{Z}$ by positing a Normalizing-Flows model for $p(Z;\alpha)$~\citep{papamakarios2021normalizing,dinh2017density}, then fitting \acrshort{MLE} to yield $\alpha$ and scale $p(Z;\alpha)$ to a specified range. We use Normalizing-Flows to fit the statistical density model $p(Z;\alpha)$ because it is simple, lightweight, provable, and provides exact log-likelihood~\citep{charpentier2022natural}. Specifically, we do \acrshort{MLE} w.r.t. the logarithm by optimizing
\begin{align}\label{method:densityestimate}
    \max_{\alpha} \{ \mathbb{E}_{\left(z = f(x)\right) \sim D_s} [ &\log(p(z;\alpha))\nonumber\\
    &:= \log(p(t;\alpha)) + \log \left | \det \left ( \frac{\partial t}{\partial z}\right ) \right | ]\},
\end{align}
 where random variable $t=s_{\alpha}(f(x))$ and $s$ is a bijective differentiable function.

Finally, to enhance generalization after combining with the likelihood value of the density function $p(Z;\alpha)$, we update the weight of classifier $g$ to normalize with the likelihood value by optimizing with the objective function as follows
\begin{equation}\label{method:2nd-optimization}
    \min_{\theta_{g}} \left \{ \mathbb{E}_{\left(z = f(x), y\right) \sim D_s} \left [ -y \log \left ( \sigma(p(z;\alpha) \cdot g(z)) \right ) \right ] \right \}.
\end{equation}

\begin{figure}[t]
\begin{algorithm}[H]
    \caption{Density-Softmax: Training and Inference}\label{alg:algorithm}
    \begin{algorithmic}
        \STATE {\bfseries Training input:} Training data $D_s$, encoder $f$, density function $p(Z;\alpha)$, classifier $g$, learning rate $\eta$, gradient-penalty coefficient $\lambda$.\;
        \FOR{$e=1\rightarrow \text{pre-train epochs}$}
            \STATE Sample $(x,y) \in D_B$ with a mini-batch $B$ from $D_s$\;
            \STATE Update $\theta_{g,f}$ as:  
            \STATE \vspace{-0.2in}\begin{align*}
                \theta_{g,f} -\eta \nabla_{\theta_{g,f}}\mathbb{E}_{(x,y)} [&-y \log \left ( \sigma(g(f(x))) \right )\\
                &+ \lambda (||\nabla_x f(x)||_2-1)^2 ]
            \end{align*}
            \vspace{-0.2in}
        \ENDFOR
        \FOR{$e=1\rightarrow \text{train-density epochs}$}
            \STATE Sample $x\in D_B$ with a mini-batch $B$ from $D_s$\;
            \STATE Update $\alpha$ as:  
            \STATE $\quad \alpha -\eta \nabla_{\alpha} \mathbb{E}_{z=f(x)} \left[ -\log(p(t;\alpha)) - \log \left | \det \left ( \frac{\partial t}{\partial z}\right ) \right |\right]$
        \ENDFOR
        \STATE {\bfseries Scale:} $p(Z;\alpha)$ to $\left(0,1\right]$\;
        \FOR{$e=1\rightarrow \text{re-optimize classifier epochs}$}
            \STATE Sample $(x,y)\in D_B$ with a mini-batch $B$ from $D_s$\;
            \STATE Update $\theta_{g}$ as: 
            \STATE $\quad \theta_{g} -\eta \nabla_{\theta_{g}}\mathbb{E}_{\left(z = f(x),y\right)} \left [ -y \log \left (  \sigma(p(z;\alpha) \cdot g(z)) \right ) \right ]$
        \ENDFOR
        \STATE {\bfseries Inference (test) input:} Test sample $x_t\in D_t$
        \STATE Make a prediction for $x_t$ by following Eq.~\ref{eq:inference}
    \end{algorithmic}
\end{algorithm}
\vspace{-0.1in}
\end{figure}

\textbf{Inference.} After completing the training process, for a new test sample $x_t \in D_t$ at the test-time, we perform prediction by combining the density function on feature space $p(z_t;\alpha)$ and classifier $g$ to predict with only a single forward pass by the following formula
\begin{equation}\label{eq:inference}
    p(y=i|x_t) = \frac{\exp(p(z_t;\alpha) 
    \cdot (z_t^\top \theta_{g_i}))}{\sum_{j=1}^K \exp(p(z_t;\alpha) \cdot (z_t^\top \theta_{g_j}))}, \forall i \in \mathcal{Y},
\end{equation}
where $z_t=f(x_t)$ is the feature of test sample $x_t$.

\begin{remark}\label{remark:efficiency} \textbf{(Computational efficiency at test-time)}
    Eq.~\ref{eq:inference} shows that the complexity of 
    Density-Softmax at test-time is similar to Deterministic~\acrshort{ERM}, i.e., $\mathcal{O}(1)$ by requiring only a single forward pass to produce the softmax probability. Its computation only needs to additionally compute $p(z_t;\alpha)$, therefore, only higher than Deterministic~\acrshort{ERM} by the additional parameter $\alpha$ and the latency of $p(z_t;\alpha)$. These numbers are often small in practice by the tables in Section~\ref{experiments}.
\end{remark}

\begin{remark} \textbf{(Disentangling aleatoric and epistemic uncertainty)} Our Eq.~\ref{eq:inference} returns both aleatoric and epistemic uncertainty. That said, it is still possible to separate these two kinds of uncertainty in our framework. In particular, using the same disentangling technique with \acrshort{DDU}~\citep{mukhoti2022deep}, we can disentangle aleatoric uncertainty by returning the softmax probability without multiplying with the density model in Eq.~\ref{eq:inference}, and epistemic uncertainty by returning the likelihood value of the density model. It is worth noticing that this additional step does not require any retraining step or additional forward passes, hence, does not harm the test-time efficiency. 
\end{remark}

The pseudo-code for our proposed Density-Softmax framework's training and inference processes is presented in Algorithm~\ref{alg:algorithm}. Due to the fact that likelihood $p(Z;\alpha) \in \left(-\infty;+\infty\right)$, the exponential function in Eq.~\ref{eq:inference} can return an undefined value if the likelihood $p(Z;\alpha) \rightarrow \infty$. Therefore, we can scale it to the range of $\left(0,1\right]$ to avoid this numerical issue. The detail for controlling the scale of likelihood is provided in Apd~\ref{apd:discussdensity}. 
\section{Theoretical Analysis of Uncertainty Estimation}
Regarding uncertainty, before presenting the main result, we first recall the definition of distribution calibration~\citep{dawid1982calibration} for the forecast:
\begin{definition}~\citep{song2019distribution,kuleshov2018accurate}\label{def:calibration}
A forecast $h$ is said to be {\bf distributional calibrated} if 
\begin{align*}
    \mathbb{P}(Y=y|h(x) = W) = w(y) \text{, } \forall y \in \mathcal{Y}, W \in \Delta_y.
\end{align*}
\end{definition}
Intuitively, this means the forecast $h$ is well-calibrated if its outputs truthfully quantify the predictive uncertainty. E.g., if we take all data points $x$ for which the model predicts $[h(x)]_y=0.3$, we expect $30\%$ of them to indeed take on the label $y$. To quantify distributional calibrated in Def.~\ref{def:calibration}, an approximate estimator of the Calibration Error~\citep{murphy1973new} in expectation was given by~\citet{naeini2015ece} and is still the most commonly used measure for a multi-class model. It is referred to as the \textbf{\acrfull{ECE}} of model $h: \mathcal{X} \rightarrow \Delta_y$ is defined as
\begin{equation}\label{eq:ece}
    \text{ECE}(h) := \sum_{m=1}^M \frac{\left | B_m \right |}{N} \left | \text{acc}(B_m) - \text{conf}(B_m)\right |,
\end{equation}
where $B_m$ is the set of sample indices whose confidence falls into $\left ( \frac{m-1}{M}, \frac{m}{M} \right ]$ in $M$ bins, $N$ is the number of samples, bin-wise mean accuracy $\text{acc}(B_m) := \frac{1}{\left | B_m \right |} \sum_{i \in B_m} \mathbb{I}(\argmax_{y \in \mathcal{Y}}[h(x_i)]_y = y_i)$, and bin-wise mean confidence $\text{conf}(B_m) := \frac{1}{\left | B_m \right |} \sum_{i \in B_m} \max_{y \in \mathcal{Y}}[h(x_i)]_y$.

In addition, let us inherit the definition of Feature Distance Awareness from~\citet{Liu2020SNGP}:
\begin{definition}\label{def:distanceaware} 
    The predictive distribution $\sigma(g(z_t))$ on test feature $z_{t}=f(x_t)$ is said \textbf{feature distance aware} if there exists $u(z_{t})$, a summary statistics of $\sigma(g(z_t))$, that quantifies model uncertainty (e.g., entropy, predictive variance, etc.) and reflects the distance between $z_{t}$ and the training data $Z_s$ w.r.t. $\left \| \cdot \right \|_\mathcal{Z}$, i.e., $u(z_{t}) := v(d(z_{t},Z_s))$, where $v$ is a monotonic function and $d(z_{t},Z_s) := \mathbb{E}\left \| z_{t} - Z_s \right \|_{\mathcal{Z}}$ is the distance between $z_{t}$ and the training data $Z_s$.
\end{definition} 

Then, we recall a Lemma when $p(Z;\alpha)$ is a Normalizing-Flows model~\citep{papamakarios2021normalizing}: 
\begin{lemma}\label{lemma:density}~\citep{charpentier2022natural}
    If $p(Z;\alpha)$ is parametrized with a Gaussian Mixture Models or a radial Normalizing-Flows, then $\lim_{d(z_{t},Z_s) \rightarrow \infty} p(z_t;\alpha) \rightarrow 0$.
\end{lemma}
Lemma~\ref{lemma:density} intuitively leads Eq.~\ref{eq:inference} to reasonable uncertainty estimation for the two limit cases of strong \acrshort{IID} and \acrshort{OOD} data. In particular, for very unlikely \acrshort{OOD} data, i.e., $d(z_{t},Z_s) \rightarrow \infty$, the prediction will go to uniform. Conversely, for very likely \acrshort{IID} data, i.e., $d(z_{t},Z_s) \rightarrow 0$, the prediction follows the in-domain predictive distribution. Now, we formally show this property below:
\begin{theorem}\label{theo:uniform_ood_normal_iid} 
    Density-Softmax provides a uniform prediction, i.e., $\sigma(p(z_{ood};\alpha) \cdot g(z_{ood})) = \mathbb{U}$ when $d(z_{ood},Z_s)\rightarrow \infty$ by $p(z_{ood});\alpha) \rightarrow 0$, and preserves the in-domain prediction, i.e., $\sigma(p(z_{iid};\alpha) \cdot g(z_{iid})) = \sigma(g(f(x_{iid})))$ when $d(z_{iid},Z_s)\rightarrow 0$ by $p(z_{iid});\alpha) \rightarrow 1$. The proof is in Apd.~\ref{proof:theo:uniform_ood_normal_iid}.
\end{theorem}
Based on these results, we next show that Density-Softmax is the solution of the minimax uncertainty risk in Eq.~\eqref{eq:minimax} and satisfies Def.~\ref{def:distanceaware} about feature distance awareness:
\begin{theorem}\label{thm:optimal}
    Density-Softmax's prediction is the optimal solution of the minimax uncertainty risk, i.e., $\sigma(p(f(X);\alpha) \cdot g(f(X))= \underset{\mathbb{P}(Y|X) \in \mathcal{P}}{\arg\inf} \left [ \sup_{\mathbb{P}^*(Y|X) \in \mathcal{P}^*} S(\mathbb{P}(Y|X), \mathbb{P}^*(Y|X)) \right ]$. The proof is in Apd.~\ref{proof:thm:optimal}.
\end{theorem}

\begin{theorem}\label{thm:distanceaware}
    The predictive distribution of Density-Softmax $\sigma(p(z=f(x);\alpha) \cdot (g \circ f(x)))$ is distance aware on the feature space $\mathcal{Z}$ by satisfying the condition in Def.~\ref{def:distanceaware}, i.e., there exists a summary statistic $u(z_{t})$ of $\sigma(p(z_t;\alpha) \cdot (g \circ z_t))$ on the new test feature $z_t=f(x_t)$ s.t., $u(z_{t}) = v(d(z_{t},Z_s))$, where $v$ is a monotonic function and $d(z_{t},Z_s) = \mathbb{E}\left \| z_{t} - Z_s \right \|_{\mathcal{Z}}$ is the distance between $z_{t}$ and the training features random variable $Z_s$.
\end{theorem}
\begin{proof}[Proof sketch] 
    Thm.~\ref{thm:distanceaware} is proved by showing  (1) $p(z_t=f(x_t);\alpha)$ is monotonically decreasing w.r.t. distance $\mathbb{E}\left \| z_{t} - Z_s \right \|_{\mathcal{Z}}$ and (2) $p(z=f(x);\alpha) \cdot g$ is distance aware. Full proof is in Apd.~\ref{proof:thm:distanceaware}.
\end{proof}
\begin{remark}\label{remark:distanceaware}
Thm.~\ref{thm:distanceaware} shows our Density-Softmax is distance aware on the feature representation $\mathcal{Z}$, i.e., its predictive probability reflects the distance between the test feature and the training set. This is a necessary condition for a \acrshort{DNN} to achieve high-quality uncertainty estimation~\citep{Liu2020SNGP}. By showing $p(Z;\alpha)$ is monotonically decreasing w.r.t. feature distance, this proves when the likelihood of $p(Z;\alpha)$ is high, our model is certain on \acrshort{IID} data, and when the likelihood of $p(Z;\alpha)$ decreases on \acrshort{OOD} data, the certainty will decrease correspondingly.
\end{remark}
Following these results, we finally present how Density-Softmax can enhance the uncertainty quality of \acrshort{DNN} with the standard softmax by reducing its over-confidence in the proposition below:
\begin{proposition}\label{prop:ECE}
    If the predictive distribution of the standard softmax $\sigma(g \circ f)$ makes $\text{acc}(B_m) \leq \text{conf}(B_m), \forall B_m, m\in [M]$, where $B_m$ is the set of sample indices whose confidence falls into $\left ( \frac{m-1}{M}, \frac{m}{M} \right ]$ in $M$ bins, then Density-Softmax $\sigma((p(f;\alpha) \cdot g)\circ f)$ can improve calibrated-uncertainty in terms of \acrshort{ECE} in Eq.~\ref{eq:ece}, i.e., $\text{ECE}(\sigma((p(f;\alpha)\cdot g)\circ f) \leq \text{ECE}(\sigma(g \circ f))$. The proof is in Apd.~\ref{proof:prop:ECE}.
\end{proposition}
\begin{remark}\label{remark:ece} 
    The condition $\text{acc}(B_m) \leq \text{conf}(B_m), \forall B_m$ is a specific case of the over-confidence for every $M$ bins. And if so, Prop.~\ref{prop:ECE} shows Density-Softmax can reduce \acrshort{ECE} of the standard softmax, which also be empirically confirmed in Fig.~\ref{fig:ece} in the next section.
\end{remark}
\begin{table*}[t!]
\begin{minipage}{1.0\textwidth}
\caption{Results for Wide~Resnet-28-10 on CIFAR-10, averaged over 10 seeds: negative log-likelihood (lower is better), accuracy (higher is better), and expected calibration error (lower is better). NLL, Acc, and ECE represent performance on \acrshort{IID} test set. cNLL, cAcc, and cECE are NLL, accuracy, and ECE averaged over \acrshort{OOD} CIFAR-10-C’s corruption types \& intensities, oNLL, oAcc, and oECE are for real-world shift CIFAR-10.1. \#Params is the number of model parameters, Latency is the milliseconds to inference per sample on RTX~A5000, ours is colored by \textcolor{blue}{blue} (lower are better). Best scores with the significant test are marked in \textbf{bold}. Details are in Tab.~\ref{tab:cifar10_results+}.}
\label{tab:cifar10_results}
\centering
\scalebox{0.78}{
\begin{tabular}{lccccccccccc}
\toprule
\textbf{Method} & \textbf{NLL($\downarrow$)} & \textbf{Acc($\uparrow$)} & \textbf{ECE($\downarrow$)} & \textbf{cNLL($\downarrow$)} & \textbf{cAcc($\uparrow$)} & \textbf{cECE($\downarrow$)} & \textbf{oNLL($\downarrow$)} & \textbf{oAcc($\uparrow$)} & \textbf{oECE($\downarrow$)} & \textbf{\#Params($\downarrow$)} & \textbf{Latency($\downarrow$)}\\
\midrule
Deterministic~ERM & 0.159 & 96.0 & 0.023 & 1.05 & 76.1 & 0.153 & 0.40 & 89.9 & 0.064 & \textbf{36.50M} & \textbf{518.12}\\
Rank-1 BNN & 0.128 & 96.3 & \textbf{0.008} & 0.84 & 76.7 & 0.080 & 0.32 & 90.4 & 0.033 & 36.65M & 1,427.67\\
Heteroscedastic & 0.156 & 96.0 & 0.023 & 1.05 & 76.1 & 0.154 & 0.38 & 90.1 & 0.056 & 36.54M & 560.43\\
SNGP & 0.134 & 96.0 & 0.007 & 0.74 & 78.5 & 0.078 & 0.43 & 89.7 & 0.064 & 37.50M & 916.26\\
DDU (w/o TS) & 0.159 & 96.0 & 0.024 & 1.06 & 76.0 & 0.153 & 0.39 & 89.8 & 0.063 & 37.60M & 954.31\\
NatPN & 0.242  & 92.8  & 0.041 & 0.89 & 73.9 & 0.121 & 0.46 & 86.3 & 0.049 & 36.58M & 601.35\\
BatchEnsemble & 0.136 & 96.3 & 0.018 & 0.97 & 77.8 & 0.124 & 0.35 & 90.6 & 0.048 & 36.58M & 1,498.01\\
Deep Ensembles & \textbf{0.114} & \textbf{96.6} & 0.010 & 0.81 & 77.9 & 0.087 & 0.28 & \textbf{92.2} & 0.025 & 145.99M &  1,520.34\\
\textbf{Density-Softmax} & 0.137 & 96.0 & 0.010 & \textbf{0.68} & \textbf{79.2} & \textbf{0.060} & \textbf{0.26} & 91.6 & \textbf{0.016} &  \color{blue}{36.58M} & \color{blue}{\textbf{520.53}}\\
\bottomrule
\end{tabular}}
\end{minipage}

\begin{minipage}{1.0\textwidth}
\caption{Results for Wide~Resnet-28-10 on CIFAR-100: cNLL, cAcc, and cECE are for CIFAR-100-C. AUPR-S and AUPR-C are AUPR for \acrshort{OOD} detection on SVHN and CIFAR-10 (higher are better). Details are in Tab.~\ref{tab:cifar100_results+}.} 
\label{tab:cifar100_results}
\centering
\scalebox{0.8}{
\begin{tabular}{lcccccccccc}
\toprule
\textbf{Method} & \textbf{NLL($\downarrow$)} & \textbf{Acc($\uparrow$)} & \textbf{ECE($\downarrow$)} & \textbf{cNLL($\downarrow$)} & \textbf{cAcc($\uparrow$)} & \textbf{cECE($\downarrow$)} & \textbf{AUPR-S($\uparrow$)} & \textbf{AUPR-C($\uparrow$)} & \textbf{\#Params($\downarrow$)} & \textbf{Latency($\downarrow$)}\\
\midrule
Deterministic~ERM & 0.875 & 79.8 & 0.086 & 2.70 & 51.4 & 0.239 & 0.882 & 0.745 & \textbf{36.55M} & \textbf{521.15}\\
Rank-1 BNN  & 0.692 & 81.3 & \textbf{0.018} & 2.24 & 53.8 & 0.117 & 0.884 & 0.797 & 36.71M & 1,448.90\\
Heteroscedastic & 0.833 & 80.2 & 0.059 & 2.40 & 52.1 & 0.177 & 0.881 & 0.752 & 37.00M & 568.17\\
SNGP & 0.805 & 80.2 & 0.020 & 2.02 & 54.6 & 0.092 & \textbf{0.923} & 0.801 & 37.50M & 926.99\\
DDU (w/o TS) & 0.877 & 79.7 & 0.086 & 2.70 & 51.3 & 0.240 & 0.890 & 0.797 & 37.60M & 959.25\\
NatPN & 1.249 & 76.9 & 0.091 & 2.97 & 48.0 & 0.265 & 0.875 & 0.768 & 36.64M & 613.44\\
BatchEnsemble & 0.690 & 81.9 & 0.027 & 2.56 & 53.1 & 0.149 & 0.870 & 0.757 & 36.63M & 1,568.77\\
Deep Ensembles & \textbf{0.666} & \textbf{82.7} & 0.021 & 2.27 & 54.1 & 0.138 & 0.888 & 0.780 & 146.22M & 1,569.23\\
\textbf{Density-Softmax} & 0.780 & 80.8 & 0.038 & \textbf{1.96} & \textbf{54.7} & \textbf{0.089} & 0.910 & \textbf{0.804} & \color{blue}{36.64M} & \color{blue}{\textbf{522.94}}\\
\bottomrule
\end{tabular}}
\end{minipage}

\begin{minipage}{1.0\textwidth}
\caption{Results for Resnet-50 on ImageNet: cNLL, cAcc, and cECE are for ImageNet-C. Details are in Tab.~\ref{tab:imagenet_results+}.}
\label{tab:imagenet_results}
\centering
\scalebox{0.8}{
\begin{tabular}{lcccccccc}
\toprule
\textbf{Method} & \textbf{NLL($\downarrow$)} & \textbf{Acc($\uparrow$)} & \textbf{ECE($\downarrow$)} & \textbf{cNLL($\downarrow$)} & \textbf{cAcc($\uparrow$)} & \textbf{cECE($\downarrow$)} & \textbf{\#Params($\downarrow$)} & \textbf{Latency($\downarrow$)}\\
\midrule
Deterministic~ERM & 0.939 & 76.2 & 0.032 & 3.21 & 40.5 & 0.103 & \textbf{25.61M} & \textbf{299.81}\\
Rank-1 BNN & 0.886 & 77.3 & 0.017 & 2.95 & 42.9 & 0.054 & 26.35M & 690.14\\
Heteroscedastic & 0.898 & 77.5 & 0.033 & 3.20 & 42.4 & 0.111 & 58.39M & 337.50\\
SNGP & 0.931  & 76.1  & \textbf{0.013} & 3.03 & 41.1 & 0.045 & 26.60M & 606.11\\
MIMO & 0.887  & 77.5  & 0.037 & 3.03 & 43.3 & 0.106 & 27.67M & 367.17\\
BatchEnsemble & 0.922 & 76.8 & 0.037 & 3.09 & 41.9 & 0.089 & 25.82M & 696.81\\
Deep Ensembles & \textbf{0.857} & \textbf{77.9} & 0.017 & 2.82 & \textbf{44.9} & 0.047 & 102.44M& 701.34\\
\textbf{Density-Softmax} & 0.885 & 77.5 & 0.019  & \textbf{2.81} & 44.6 & \textbf{0.042} & \color{blue}{25.88M} & \color{blue}{\textbf{299.90}}\\
\bottomrule
\end{tabular}}
\end{minipage}
\end{table*}

\section{Experiments}\label{experiments}
\begin{figure}[ht!]
    \centering
    \includegraphics[width=1.0\linewidth]{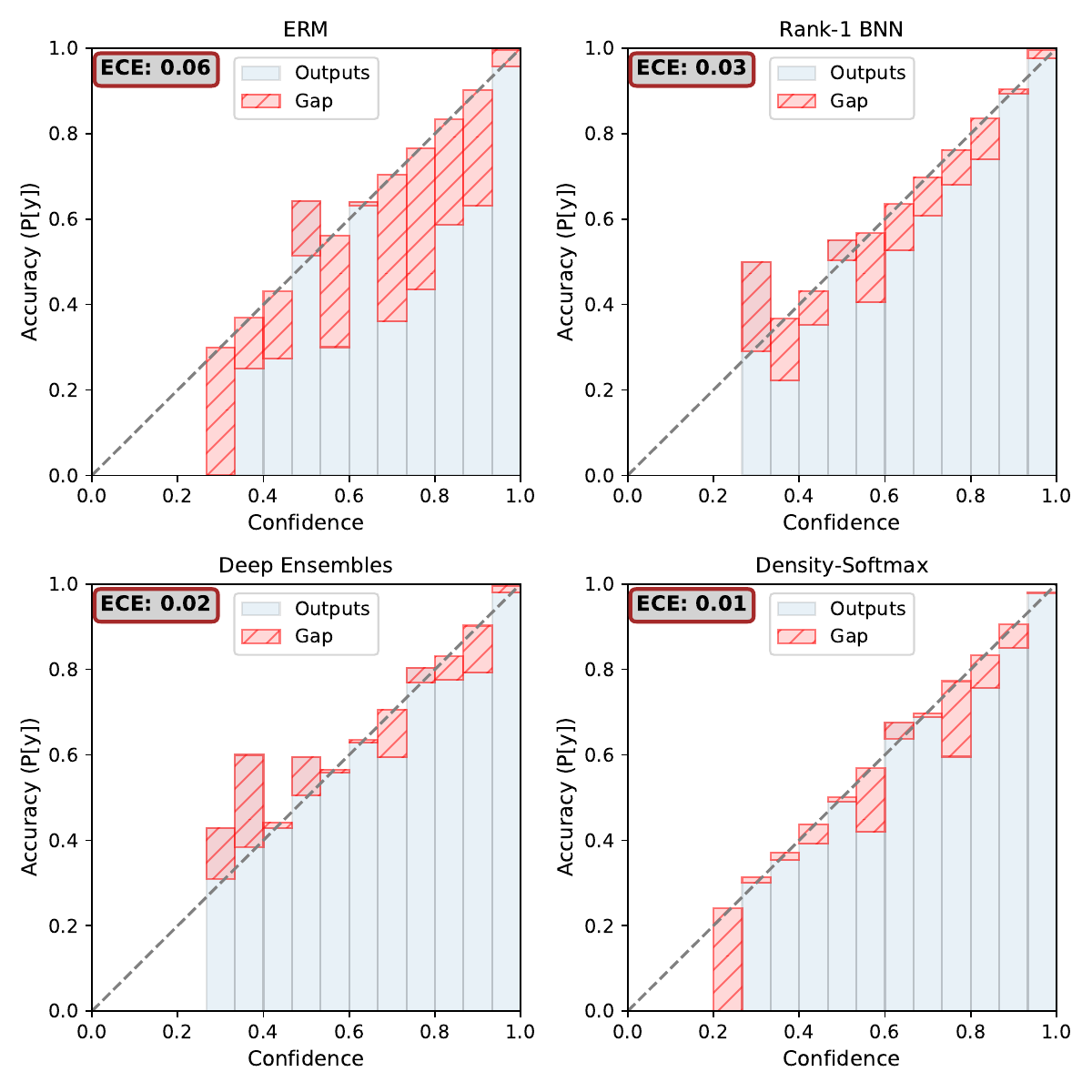}
    \vspace{-0.2in}
    \caption{Reliability diagram of Density-Softmax v.s. different approaches, trained on CIFAR-10, test on CIFAR-10.1 (v6). Density-Softmax is better-calibrated than others. Details are in Apd.~\ref{apd:results_calib}.}
    \label{fig:ece}
\end{figure}
We follow experimental settings based on the source code of~\citet{nado2021uncertainty}. Baseline details are in Apd.~\ref{apd:baseline}. Datasets, evaluation metrics, training details, architectures and hyper-parameters, and source code and computing system details are in Apd.~\ref{apd:implementation}. Our source code is available at \href{https://github.com/Angie-Lab-JHU/density_softmax}{https://github.com/Angie-Lab-JHU/density\_softmax}.

\subsection{Robustness Performance}
\textbf{Density-Softmax achieves competitive robust generalization with \acrshort{SOTA}.} Tab.~\ref{tab:cifar10_results}, \ref{tab:cifar100_results}, and \ref{tab:imagenet_results} show benchmark results across shifted datasets. We observe Density-Softmax achieve a competitive result with the \acrfull{SOTA} in \acrshort{NLL} and accuracy under distribution shifts. Specifically, our method has the lowest \acrshort{NLL} and highest accuracy with $0.68, 79.2\%$, $1.96, 54.7\%$, and $2.81, 44.6\%$ respectively in the corrupted \acrshort{OOD} datasets CIFAR-10-C, CIFAR-100-C, and ImageNet-C. Regarding the real-world shift CIFAR-10.1, it also achieves the lowest \acrshort{NLL} with $0.26$, and $91.6\%$ in accuracy, higher than other methods and only lower than Deep Ensembles by $0.6\%$. Although there is a trade-off between \acrshort{IID} and \acrshort{OOD} performance by the Lipschitz constraint in Apd.~\ref{apd:additional_ablation}, it is also worth noticing that with an appropriate $\lambda$, our method can still preserve a high accuracy in \acrshort{IID} at the same time and outperforms many baselines. E.g., it achieves $77.5\%$ in ImageNet, higher than Deterministc~\acrshort{ERM}, Rank-1~\acrshort{BNN}, \acrshort{SNGP}, BatchEnsemble, etc. More details about benchmark comparison are in Apd.~\ref{apd:results_benchmark}.

\begin{figure*}[t!]
    \centering
    \setlength{\tabcolsep}{1pt}
    \begin{tabular}{ccc}
         \includegraphics[width=0.33\textwidth]{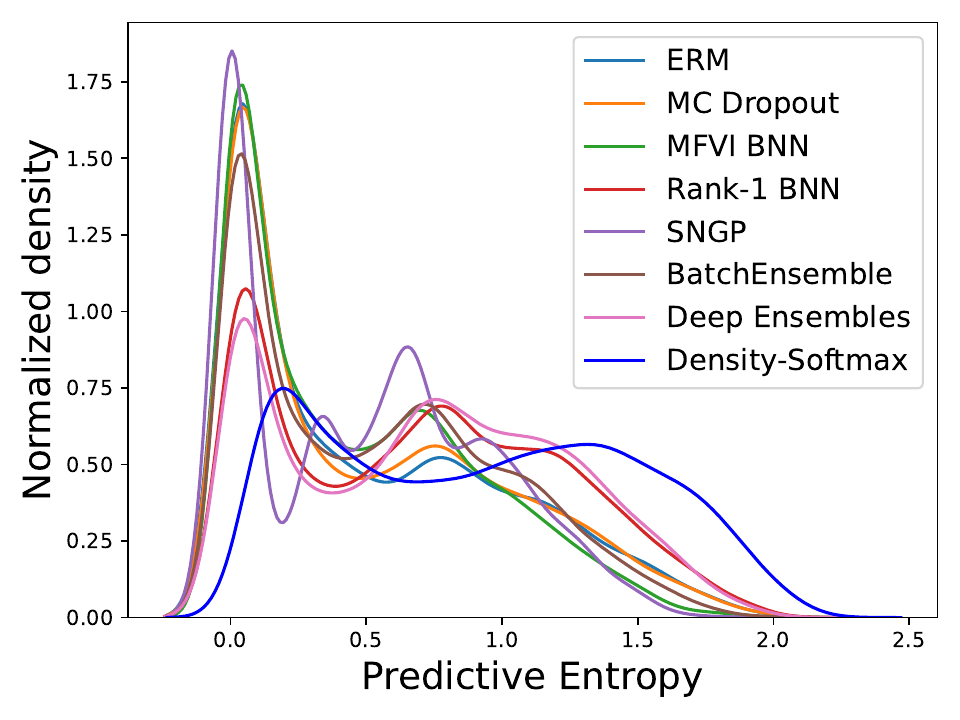}&
         \includegraphics[width=0.33\textwidth]{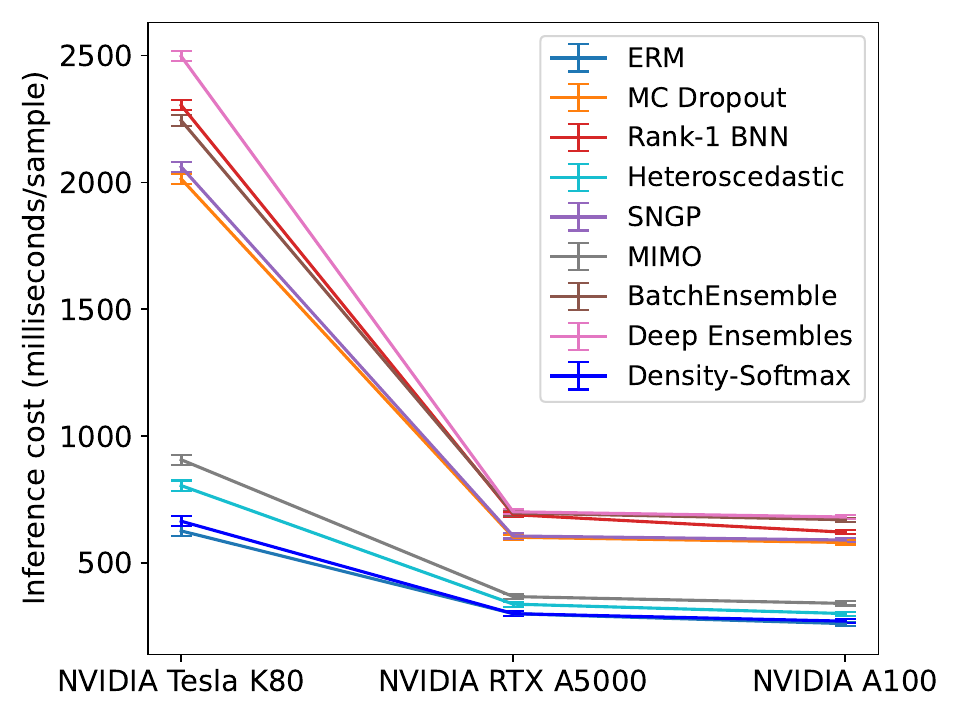}&
         \includegraphics[width=0.33\textwidth]{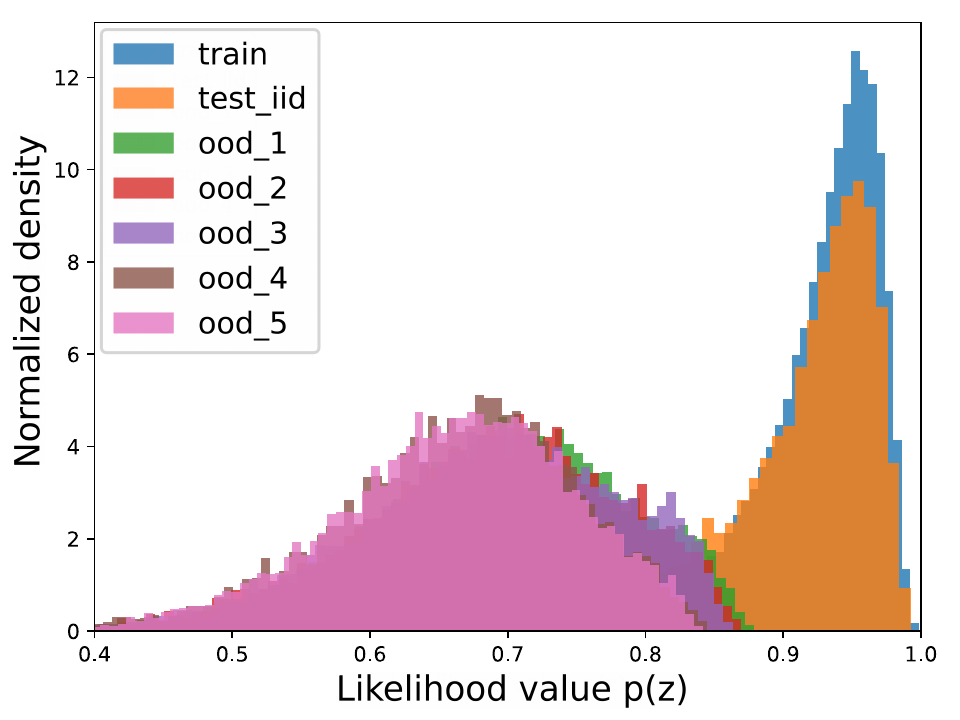}\\
         \small (a) & (b) & (c)
    \end{tabular}
    \vspace{-0.08in}
    \caption{(a) PDF plot of predictive entropy $\mathrm{H}(p(y|x))$ for the semantic shift. Density-Softmax provides the highest entropy with high density for \acrshort{OOD}; (b) Inference cost comparison at test-time on ImageNet. Our model consistently outperforms \acrshort{SOTA} across modern GPU architectures; (c) Histogram of $p(z;\alpha)$'s likelihood. \textcolor{blue}{Blue} represents on CIFAR-10 train, \textcolor{orange}{Orange} is \acrshort{IID} test, \textcolor{green}{Green}, \textcolor{red}{Red}, \textcolor{purple}{Purple}, \textcolor{brown}{Brown}, \textcolor{pink}{Pink} are \acrshort{OOD} from 1-5 shift levels on CIFAR-10-C. It produces high values on \acrshort{IID} and lower values on \acrshort{OOD} w.r.t. intensity levels.}
    \label{fig:main}
\end{figure*}

\subsection{Uncertainty Estimation Performance}
\textbf{Density-Softmax achieves competitive uncertainty performances with \acrshort{SOTA}.} From the tables, we also observe our model has a competitive uncertainty quality and even sometimes outperforms the \acrshort{SOTA} like Rank-1~\acrshort{BNN}, \acrshort{SNGP}, and Deep Ensembles, especially under \acrshort{OOD} settings. E.g., it achieves the lowest cECE with $0.060$ in CIFAR-10-C, $0.089$ in CIFAR-100-C, and $0.042$ in ImageNet-C. Similarly, it has the best \acrshort{AUPR}-C with $0.804$ in CIFAR-10 and $0.016$ oECE in CIFAR-10.1.\\
To take a closer look at the calibration, we visualize reliability diagrams based on the \acrshort{ECE} in Fig.~\ref{fig:ece}. We observe that Density-Softmax is better calibrated than other methods in the real-world shifted \acrshort{OOD} test set. E.g., compared to Deterministic~\acrshort{ERM}, our model is less over-confident, confirming Prop.~\ref{prop:ECE}. Meanwhile, compared to Rank-1~\acrshort{BNN} and Deep Ensembles, it is less under-confidence. Calibration details with reliability diagrams in both \acrshort{IID} and \acrshort{OOD} are in Apd.~\ref{apd:results_calib}.

\textbf{Density-Softmax achieves distance awareness.} From Fig.~\ref{fig:2dmoon}, we observe that our model achieves distance awareness by having uniform class probability and high uncertainty value on \acrshort{OOD} data on the two moons dataset, confirming Thm.~\ref{thm:distanceaware}. Meanwhile, Deterministic~\acrshort{ERM}, \acrshort{MC}~Dropout, Rank-1~\acrshort{BNN}, and Ensembles can not provide distance awareness as they provide no informative variance for \acrshort{OOD} data. Similar observations are in Apd.~\ref{apd:results_toy} with different uncertainty metrics and datasets.

\textbf{Density-Softmax produces high entropy on \acrshort{OOD} and low entropy on \acrshort{IID} data under semantic shift.} Fig.~\ref{fig:main}~(a) compares the density of predictive entropy between different methods trained on CIFAR-10 and tested on CIFAR-100. Because this is the semantic shift~\citep{tran2022plex}, we would expect the model to provide a high entropy value. Indeed, we observe that Density-Softmax achieves the highest entropy with a high-density value. In Fig.~\ref{fig:apd:results_entropy} in Apd.~\ref{apd:results_entropy}, we also observe that our model preserves low entropy with high-density value on \acrshort{IID} data, confirming the hypothesis that our framework can enhance uncertainty quantification. More details about the histograms are in Apd.~\ref{apd:results_entropy}.

\textbf{Density-Softmax outperforms \acrshort{SOTA} in inference speed.} Our method only requires one forward pass to make a prediction, so it outperforms other \acrshort{SOTA} in terms of inference speed. For every dataset with different backbones, we observe that Density-Softmax achieves almost the same latency with Deterministic~\acrshort{ERM}. In particular, it takes less than $525$ and $300$ ms/sample in Wide~Resnet-28-10 and Resnet-50 for an inference on RTX~A5000. To make a further comparison on test-time latency, we compare the running time on 3 modern GPU architectures in Fig.~\ref{fig:main}~(b). We observe that our model consistency outperforms \acrshort{SOTA}, especially for lower computational hardware like NVIDIA Tesla~K80.\\
Having a similar latency, the tables show our model is also more reliable than Deterministic~\acrshort{ERM} by always achieving a lower NLL, ECE, and higher accuracy. Therefore, these results suggest that Density-Softmax could be a potential deterministic approach for uncertainty and robustness in real-time applications.

\begin{figure*}[t!]
    \centering
    \setlength{\tabcolsep}{1pt}
    \begin{tabular}{cccc}
         \includegraphics[width=0.25\textwidth]{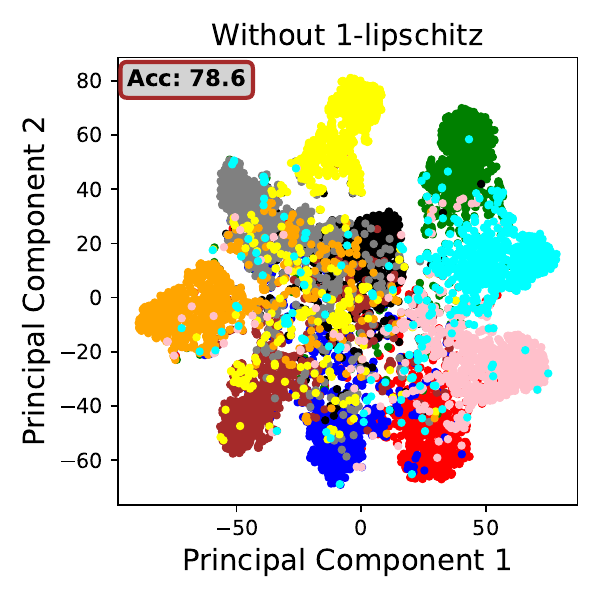}&
         \includegraphics[width=0.25\textwidth]{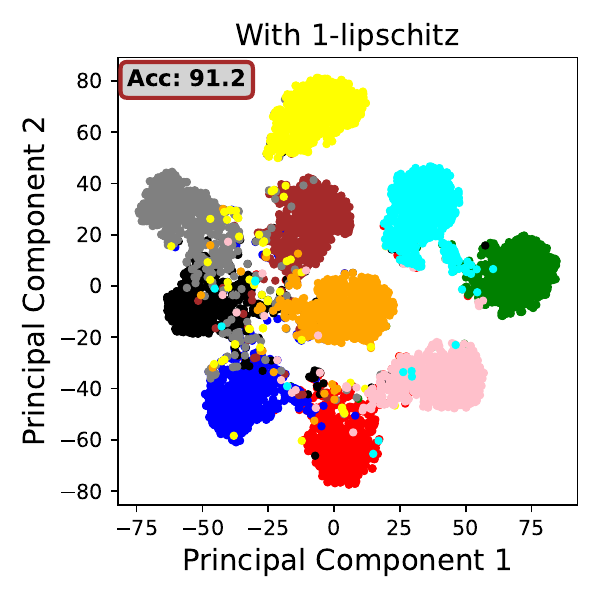}&
         \includegraphics[width=0.25\textwidth]{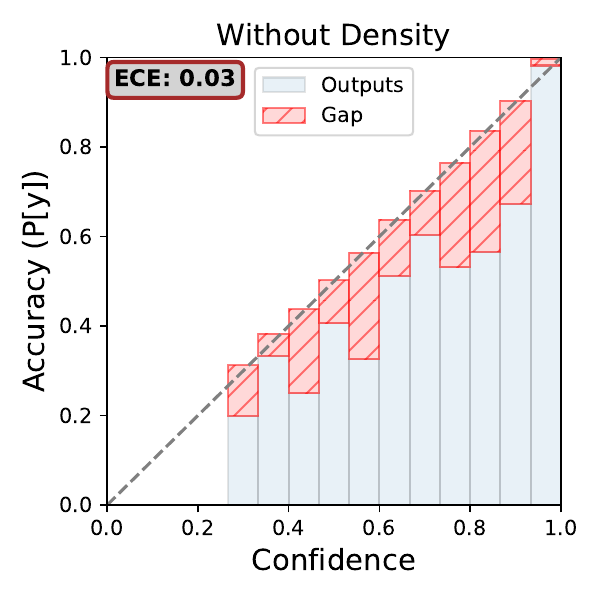}&
         \includegraphics[width=0.25\textwidth]{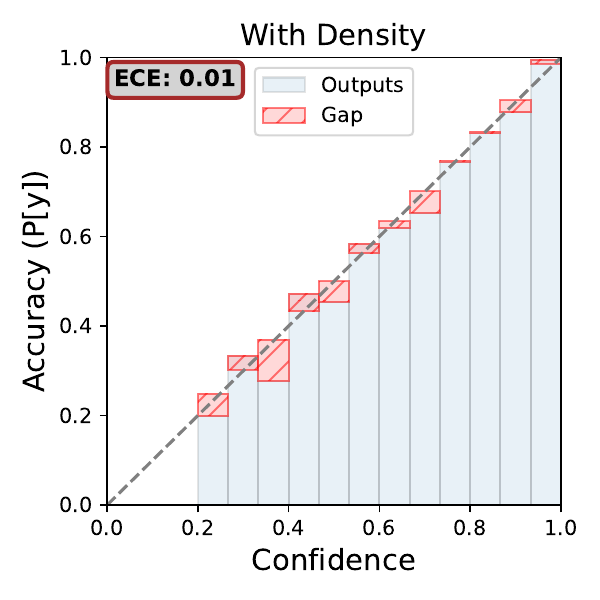}\\
         \small (a) & (b) & (c) & (d)
    \end{tabular}
    \vspace{-0.08in}
    \caption{Feature visualizations comparison between models with \& without 1-Lipschitz constraint on CIFAR-10-C (a \& b), reliability diagrams between models with \& without the density-function on CIFAR-10 (c \& d).}
    \vspace{-0.04in}
    \label{fig:ablation}
\end{figure*}

\textbf{Density-Softmax outperforms \acrshort{SOTA} in storage requirements.} From the tables, the parameters and latency of $p(Z;\alpha)$ can be measured by the minus of ours to Deterministic~\acrshort{ERM} by Re.~\ref{remark:efficiency}. We observe that Density-Softmax is very lightweight with less than $36.65$M parameters in Wide~Resnet-28-10 and $25.88$M in Resnet-50. These numbers are lower than other \acrshort{SOTA} baselines, e.g., Rank-1~\acrshort{BNN}, Heteroscedastic, \acrshort{SNGP}, \acrshort{MIMO}, and Deep Ensembles with Resnet-50 on ImagetNet.\\
In summary, combined with the latency performance, Density-Softmax outperforms other \acrshort{SOTA} approaches in computational efficiency at the test-time. In particular, our model is less than Deep Ensembles by 4 times in the number of parameters. Regarding the latency, Density-Softmax is also much faster than other \acrshort{SOTA} baselines across different hardware architectures.

\subsection{Ablation study: Analysis of Density Softmax's components}
\textbf{How does Density Softmax work?} Our framework is a combination of 2 main components: (1) the Lipschitz-constrained $f$ and (2) the density function $p(Z;\alpha)$. To test their importance, we make a comparison between architectures with and without each in our model. From Fig.~\ref{fig:ablation}, we observe that without the 1-Lipschitz regularization, the model has a worse feature representation than (1), e.g., leading to a drop in the accuracy from $91.2\%$ to $78.6\%$ on the CIFAR-10-C with defocus\_blur\_5. Similarly, without $p(Z;\alpha)$, the model has a worse uncertainty quality than (2), e.g., causing an increase in the \acrshort{ECE} from $0.01$ to $0.03$ on CIFAR-10. These results prove that the Lipschitz-constrained $f$ helps our model improve robustness, while the density function $p(Z;\alpha)$ enhances the uncertainty quantification. Last but not least, we discover that Lipschitz-constrained $f$ not only improves the robustness but also the uncertainty quality by observing a worse \acrshort{ECE} performance without the Lipschitz constraint, i.e., $\lambda = 0$ in Tab.~\ref{tab:lamda}.

\textbf{Our density model can capture distributional shifts.} Density-Softmax estimates the density on the low-dimensional feature space~$\mathcal{Z}$, which is fixed after the first step of Algorithm~\ref{alg:algorithm}. This feature structure is low-dimensional, task-specific, and encodes meaningful semantic features~\citep{bui2024density,charpentier2022natural}. Therefore, this feature space~$\mathcal{Z}$ is much simpler to estimate when compared to the complex image pixels space~$\mathcal{X}$. We visualize the likelihood histogram of our Normalizing-Flows density function across training, \acrshort{IID} testing, and \acrshort{OOD} sets in Fig.~\ref{fig:main}~(c). We observe that our density function provides a high likelihood for \acrshort{IID} while low values for \acrshort{OOD} test set. Importantly, when the shift intensity increases, the likelihood also decreases, showing that our model can reduce certainty correspondingly.
\section{Conclusion and Discussion}
Despite showing success in reliable \acrshort{DNN}, sampling-based methods like Deep Ensembles and \acrshort{BNN} suffer from huge computational burdens at test-time. To tackle this challenge, we introduce Density-Softmax, a sampling-free approach to improve uncertainty and robustness via a combination of a feature density function from the Lipschitz-constrained feature extractor with the softmax layer. We complement this algorithm with a theoretical analysis establishing guarantees on the 1-Lipschitz constraint, solution of the minimax uncertainty risk, distance awareness on feature space, and reducing over-confidence of the standard softmax. Empirically, we find our proposed framework achieves competitive results with \acrshort{SOTA} methods in uncertainty and robustness while outperforming them significantly in terms of memory and inference cost at test-time. We hope that our work will be an option for developers to try in real-world applications and inspire researchers to further progress in the area of improving the \acrshort{DNN} model efficiency and reliability.

\textbf{Density model performance in practice.} The uncertainty quality of our method depends on the density function. Our results show if the likelihood on test \acrshort{OOD} feature is lower than \acrshort{IID} set, then Density-Softmax can reduce the over-confidence of the standard softmax. Yet, it can be a risk that our model might not fully capture the real-world complexity by estimating density model is not always trivial in practice~\citep{nalisnick2019doDGM,charpentier2022natural}. 
    
\textbf{Training cost.} Despite showing success in test-time efficiency,
we also ask users to be cautious about the longer training-time of our Density-Softmax than Deterministic~\acrshort{ERM} in practice (details are in Apd.~\ref{apd:additional_ablation}).

\textbf{Remediation.} Given the aforementioned limitations, we encourage people who extend our work to: (1) proactively confront the model design and parameters to desired behaviors in real-world use cases; (2) be aware of the training challenge and prepare enough time and resources (e.g., setting enough GPU servers, training on GPU cloud services, etc.) to pre-train our framework in practice.

\textbf{Future work.} Future work includes providing uncertainty estimates for pre-trained large models (details are in Apd.~\ref{apd:additional_discuss}), developing new techniques to avoid computing the Jacobian matrix at training-time, improving estimation techniques to enhance the quality of the density function, and continuing to reduce the number of parameters to deploy this framework in real-world systems.  

\section*{Impact Statement}
Uncertainty \& robustness are critical problems in trustworthy AI. There has been growing interest in using sampling-based methods to ensure \acrshort{DNN} are robust and reliable. Challenges often arise when deploying such models in real-world applications. In this regard, Density-Softmax significantly improves test-time efficiency while preserving reliability. This could be particularly beneficial in high-stake applications (e.g., healthcare, finance, policy decision-making, etc.), where the trained model needs to be deployed and inference on low-resource hardware or real-time response software. 

\section*{Acknowledgments}
This work is partially supported by the JHU-Amazon AI2AI faculty award, the Discovery Award of the Johns Hopkins University, and a seed grant from JHU Institute of Assured Autonomy.

\bibliography{refs}
\bibliographystyle{icml2024}
\printglossary[type=\acronymtype,nonumberlist]
\printglossary

\appendix
\onecolumn
\noindent\rule{\textwidth}{1pt}
\section*{\Large\centering{Density-Softmax: Efficient Test-time Model for Uncertainty Estimation and Robustness under Distribution Shifts\\(Supplementary Material)}} 
\noindent\rule{\textwidth}{1pt}

\textbf{Reproducibility.} Our code inherits from \href{https://github.com/google/uncertainty-baselines/tree/main}{the uncertainty-baselines codebase}, so our reported results could be referred from this page. In Apd.~\ref{apd:exp}, we provide detailed information about our experiments, including baseline in Apd.~\ref{apd:baseline}, implementation in Apd.~\ref{apd:implementation}, and additional results in Apd.~\ref{apd:results}. The additional results in Apd.~\ref{apd:results} contain results for the toy dataset in Apd.~\ref{apd:results_toy}, results for the benchmark dataset in Apd.~\ref{apd:results_benchmark}, uncertainty details about calibration in Apd.~\ref{apd:results_calib}, and about predictive entropy in Apd.~\ref{apd:results_entropy}. In Apd.~\ref{apd:discussion}, we make further discussions about our method, including density estimation and likelihood value implementation in Apd.~\ref{apd:discussdensity}, distance preserving and 1-Lipschitz constraint in Apd.~\ref{apd:distance preserving}, additional ablation study about the gradient-penalty and training cost in Apd.~\ref{apd:additional_ablation}, and additional discussion about our method in Apdx.~\ref{apd:additional_discuss}. Finally, in Apd.~\ref{apd:proof}, we provide the proofs for all the results in the main paper, including proof of Thm.~\ref{theo:uniform_ood_normal_iid} in Apd.~\ref{proof:theo:uniform_ood_normal_iid}, proof of Thm.~\ref{thm:optimal} in Apd.~\ref{proof:thm:optimal}, proof of Thm.~\ref{thm:distanceaware} in Apd.~\ref{proof:thm:distanceaware}, and proof of Prop.~\ref{prop:ECE} in Apd.~\ref{proof:prop:ECE}. 
\section{Proofs}\label{apd:proof}
In this appendix, we provide proof of the theoretical results from the main paper.
\subsection{Proof of Theorem~\ref{theo:uniform_ood_normal_iid}}\label{proof:theo:uniform_ood_normal_iid}
The proof contains two parts. The first part shows Density-Softmax provides a uniform prediction when $d(z_{ood},Z_s)\rightarrow \infty$. The second part shows Density-Softmax preserves the in-domain prediction when $d(z_{iid},Z_s)\rightarrow 0$.

\textit{\underline{Part (1). Density-Softmax provides a uniform prediction when $d(z_{ood},Z_s)\rightarrow \infty$}}:

\begin{proof} Consider the non-uniform in predictive distribution $\sigma(g(f(x_{ood}))) \neq \mathbb{U}$, we have
\begin{align}
    p(y=i|x_{ood}) = \frac{\exp(z_{ood}^\top \theta_{g_i})}{\sum_{j=1}^K \exp(z_{ood}^\top \theta_{g_j})} \neq \frac{1}{K}, \forall i \in \mathcal{Y},
\end{align}
where $z_{ood}=f(x_{ood})$ is the latent presentation for test sample $x_{ood}$, $K$ is the total of the number of categorical in the discrete label space $\mathcal{Y}$.

Let us rewrite the Density-Softmax predictive distribution $\sigma(p(z_{ood};\alpha) \cdot g(z_{ood}))$ by
\begin{align}
p(y=i|x_{ood}) = \frac{\exp(p(z_{ood};\alpha) \cdot (z_{ood}^\top \theta_{g_i})}{\sum_{j=1}^K \exp(p(z_{ood};\alpha) \cdot (z_{ood}^\top \theta_{g_j}))}, \forall i \in \mathcal{Y}.
\end{align}

Using Lemma~\ref{lemma:density}, we have $\lim_{d(z_{t},Z_s) \rightarrow \infty} p(z_t;\alpha) \rightarrow 0$, i.e., if $d(z_{ood},Z_s) \rightarrow \infty$ then $p(z_{ood};\alpha) \rightarrow 0$. Therefore, we obtain
\begin{align}
\lim_{p(z_{ood}; \alpha) \rightarrow 0} \exp(p(z_{ood};\alpha) \cdot (z_{ood}^\top \theta_{g_i})) = e^{0} = 1, \forall i \in \mathcal{Y}.
\end{align}

Since $\exp(p(z_{ood};\alpha) \cdot (z_{ood}^\top \theta_{g_j})) = 1,\forall j \in \mathcal{Y}$ when $p(z_{ood}; \alpha) \rightarrow 0$, then
\begin{align}
p(y=i|x_{ood}) &= \frac{\exp(p(z_{ood};\alpha) \cdot (z_{ood}^\top \theta_{g_i}))}{\sum_{j=1}^K \exp(p(z_{ood};\alpha) \cdot (z_{ood}^\top \theta_{g_j}))} = \frac{1}{\sum_{i=1}^K 1} = \frac{1}{K}, \forall i \in \mathcal{Y}.
\end{align}
 
As a consequence, when $d(z_{ood},Z_s) \rightarrow \infty$, we obtain the conclusion: $\sigma(p(z_{ood};\alpha) \cdot g(z_{ood})) = \mathbb{U}$, where $\mathbb{U}$ stands for uniform distribution of Theorem~\ref{theo:uniform_ood_normal_iid}.

\textit{\underline{Part (2). Density-Softmax preserves the in-domain prediction when $d(z_{iid},Z_s)\rightarrow 0$}}:

Consider Density-Softmax predictive distribution $\sigma(p(z_{iid};\alpha) \cdot g(z_{iid}))$, we have
\begin{align}
p(y=i|x_{iid}) = \frac{\exp(p(z_{iid};\alpha) \cdot (z_{iid}^\top \theta_{g_i}))}{\sum_{j=1}^K \exp(p(z_{iid};\alpha) \cdot (z_{iid}^\top \theta_{g_j}))}, \forall i \in \mathcal{Y}.
\end{align}

Due to the likelihood value of $p(Z;\alpha)$ is scale in the range of $(0,1]$, we have $\lim_{d(z_t,Z_s)\rightarrow 0} p(z_t;\alpha) \rightarrow 1$, i.e., if $d(z_{iid},Z_s)\rightarrow 0$ then $p(z_{iid}; \alpha) \rightarrow 1$. Therefore, we obtain
\begin{align}
\lim_{p(z_{iid}; \alpha) \rightarrow 1} \exp(p(z_{iid};\alpha) \cdot (z_{iid}^\top \theta_{g_i})) = \exp(z_{iid}^\top \theta_{g_i}), \forall i \in \mathcal{Y}.
\end{align}

Since $\exp(p(z_{iid};\alpha) \cdot (z_{iid}^\top \theta_{g_i})) = \exp(z_{iid}^\top \theta_{g_i}),\forall i \in \mathcal{Y}$ when $p(z_{iid}; \alpha) \rightarrow 1$, then
\begin{align}
p(y=i|x_{iid}) = \frac{\exp(p(z_{iid};\alpha) \cdot (z_{iid}^\top \theta_{g_i}))}{\sum_{j=1}^K \exp(p(z_{iid};\alpha) \cdot (z_{iid}^\top \theta_{g_j}))} = \frac{\exp(z_{iid}^\top \theta_{g_i})}{\sum_{j=1}^K \exp(z_{iid}^\top \theta_{g_j})}, \forall i \in \mathcal{Y}.
\end{align}

As a consequence, when $d(z_{iid},Z_s) \rightarrow 0$, we obtain the conclusion: $\sigma(p(z_{iid};\alpha) \cdot g(z_{iid})) = \sigma(g(f(x_{iid})))$ of Theorem~\ref{theo:uniform_ood_normal_iid}.
\end{proof}

\subsection{Proof of Theorem~\ref{thm:optimal}}\label{proof:thm:optimal}
\begin{proof}
This proof is based on the following provable Lemma of~\citet{Liu2020SNGP}:
\begin{lemma}~\citep{Liu2020SNGP,Grunwald2004game}\label{proof:corol:lemma} \textbf{(The uniform distribution $\mathbb{U}$ is the optimal for minimax Bregman score in $x\notin \mathcal{X}_{iid}$).}
    Consider the Bregman score~\citep{parry2012proper} as follows
    \begin{align*}
        s(p,p^*|x)= \sum_{k=1}^K\left\{\left[p^*(y_k|x)-p(y_k|x)\right]\psi'(p^*(y_k|x))-\psi(p^*(y_k|x))\right\},
    \end{align*}
    where $\psi$ is a strictly concave and differentiable function. Bregman score reduces to the log score when $\psi(p)=p\log(p)$, and reduces to the Brier score when $\psi(p)=p^2-\frac{1}{K}$.
\end{lemma}

So, let consider the strictly proper scoring rules $S(\mathbb{P}(Y|X),\mathbb{P^*}(Y|X))$~\citep{gneiting2007strictlyproper, bröcker2009decomposeproper}, since $\mathbb{P}(Y|X)$ is predictive, and $\mathbb{P^*}(Y|X)$ is the data-generation distribution, then for $\mathcal{X}_{ood}=\mathcal{X}/\mathcal{X}_{idd}$, using the result from~\citet{Liu2020SNGP}, we have
\begin{align}
    S(\mathbb{P},\mathbb{P^*}) &= \mathbb{E}_{x \sim X}(s(p,p^*|x)) = \int_{\mathcal{X}}s(p,p^*|x)p^*(x)dx\\
    &=\int_{\mathcal{X}}s(p,p^*|x)\left[p^*(x|x\in \mathcal{X}_{iid})p^*(x\in \mathcal{X}_{iid})+p^*(x|x\in \mathcal{X}_{ood})p^*(x\in \mathcal{X}_{ood})\right]dx\\
    &=\underbrace{\mathbb{E}_{x \sim X_{iid}}(s(p,p^*|x))}_{S_{iid}(\mathbb{P},\mathbb{P^*})}p^*(x\in \mathcal{X}_{iid})+\underbrace{\mathbb{E}_{x \sim X_{ood}}(s(p,p^*|x))}_{S_{ood}(\mathbb{P},\mathbb{P^*})}p^*(x\in \mathcal{X}_{ood}).
\end{align}

Therefore, since $S_{iid}(\mathbb{P},\mathbb{P^*})$ and $S_{ood}(\mathbb{P},\mathbb{P^*})$ has disjoint support, we have the minimax uncertainty risk, i.e., $\inf_{\mathbb{P}(Y|X) \in \mathcal{P}} \left [ \sup_{\mathbb{P^*}(Y|X) \in \mathcal{P*}} S(\mathbb{P}(Y|X), \mathbb{P^*}(Y|X)) \right ]$ equivalents to
\begin{align}
\inf_{\mathbb{P}}\sup_{\mathbb{P}^*}S(\mathbb{P},\mathbb{P^*})&=\inf_{\mathbb{P}}\left[\sup_{\mathbb{P}^*}\left[S_{iid}(\mathbb{P},\mathbb{P^*})\right] \cdot \mathbb{P}^*(X_{iid})+\sup_{\mathbb{P}^*}\left[S_{ood}(\mathbb{P},\mathbb{P^*})\right] \cdot \mathbb{P}^*(X_{ood})\right]\\
&=\inf_{\mathbb{P}}\sup_{\mathbb{P}^*}\left[S_{iid}(\mathbb{P},\mathbb{P^*})\right] \cdot \mathbb{P}^*(X_{iid})+\inf_{\mathbb{P}}\sup_{\mathbb{P}^*}\left[S_{ood}(\mathbb{P},\mathbb{P^*})\right] \cdot \mathbb{P}^*(X_{ood}).
\end{align}

Due to $\mathbb{P}(Y|X_{iid})$ is the predictive distribution learned from data, we obtain $\inf_{\mathbb{P}}\sup_{\mathbb{P}^*}\left[S_{iid}(\mathbb{P},\mathbb{P^*})\right]$ is fixed and Density-Softmax satisfy by it is the model’s predictive distribution learned from \acrshort{IID} data. On the other hand, we have
\begin{align}
\sup_{\mathbb{P}^*\in \mathcal{P}^*}\left[S_{ood}(\mathbb{P},\mathbb{P^*})\right]=\mathbb{E}_{x\sim X_{ood}}\sup_{p^*}\left[s(p,p^*|x)\right]p(x).
\end{align}

Applying the result from Lemma~\ref{proof:corol:lemma} and combining with the result for \acrshort{OOD} data from Theorem~\ref{theo:uniform_ood_normal_iid} and Proof~\ref{proof:theo:uniform_ood_normal_iid}, we obtain
\begin{align}
\sigma(p(f(x_{ood});\alpha) \cdot g(f(x_{ood}))=\mathbb{U}=\underset{\mathbb{P}\in \mathcal{P}}{\arg\inf}\sup_{\mathbb{P}^*\in \mathcal{P}^*}\left[S_{ood}(\mathbb{P},\mathbb{P}^*)\right].
\end{align}

As a consequence, we obtain the conclusion: Density-Softmax's prediction is the optimal solution of the minimax uncertainty risk, i.e,
\begin{align}
    \sigma(p(f(X);\alpha) \cdot g(f(X))= \underset{\mathbb{P}(Y|X) \in \mathcal{P}}{\arg\inf} \left [ \sup_{\mathbb{P^*}(Y|X) \in \mathcal{P*}} S(\mathbb{P}(Y|X), \mathbb{P^*}(Y|X)) \right ]
\end{align}
of Theorem~\ref{thm:optimal}.
\end{proof}

\subsection{Proof of Theorem~\ref{thm:distanceaware}}\label{proof:thm:distanceaware}
\begin{proof}
The proof contains three parts. The first part shows density function $p(z_t;\alpha)$ is monotonically decreasing w.r.t. distance function $\mathbb{E}\left \| z_{t} - Z_s \right \|_{\mathcal{Z}}$. The second part shows the metric $u(x_t)$ is maximized if $p(z_t;\alpha) \rightarrow 0$. The third part shows $u(x_t)$ monotonically decreasing w.r.t. $p(z_t;\alpha)$ on the interval $\left ( 0, 1\right ]$.

\textit{\underline{Part (1). The monotonic decrease of density function $p(z_t;\alpha)$ w.r.t. distance function $\mathbb{E}\left \| z_{t} - Z_s \right \|_{\mathcal{Z}}$}}: Consider the probability density function $p(z_t;\alpha)$ follows Normalizing-Flows which output the Gaussian distribution with mean (median) $\mu$ and standard deviation $\sigma$, then we have
\begin{align}
p(z_t;\alpha) = \frac{1}{\sigma \sqrt{2\pi}} \exp\left(\frac{-1}{2}\left( \frac{z_t - \mu}{\sigma} \right)^2\right).
\end{align}

Take derivative, we obtain
\begin{align}\label{eq:proof:pdf}
    \frac{d}{d z_t}p(z_t;\alpha) &= \left[ \frac{-1}{2}\left( \frac{z_t - \mu}{\sigma} \right)^2\right]' p(z_t;\alpha)
    = \frac{\mu - z_t}{\sigma^2} p(z_t;\alpha) \Rightarrow
     \begin{cases} 
\frac{d}{d z_t}p(z_t;\alpha) > 0 &\text{ if } z_t < \mu,\\ 
\\
\frac{d}{d z_t}p(z_t;\alpha) = 0 &\text{ if } z_t = \mu,\\
\\
\frac{d}{d z_t}p(z_t;\alpha) < 0 &\text{ if } z_t > \mu.
    \end{cases}
\end{align}

Consider the distance function $\mathbb{E}\left \| z_{t} - Z_s \right \|_{\mathcal{Z}}$ follows the absolute norm, then we have
\begin{align}
    \mathbb{E}\left \| z_{t} - Z_s \right \|_{\mathcal{Z}} = \mathbb{E}\left( \left| z_{t} - Z_s \right| \right) = \int_{-\infty}^{z_t} \mathbb{P}(Z_s \leq t) dt + \int_{z_t}^{+\infty} \mathbb{P}(Z_s \geq t) dt.
\end{align}

Take derivative, we obtain
\begin{align}\label{eq:proof:norm}
    \frac{d}{d z_t}\mathbb{E}\left \| z_{t} - Z_s \right \|_{\mathcal{Z}} = \mathbb{P}(Z_s \leq z_{t}) - \mathbb{P}(Z_s \geq z_{t}) \Rightarrow
     \begin{cases} 
\frac{d}{d z_t}\mathbb{E}\left \| z_{t} - Z_s \right \|_{\mathcal{Z}} < 0 &\text{ if } z_t < \mu,\\ 
\\
\frac{d}{d z_t}\mathbb{E}\left \| z_{t} - Z_s \right \|_{\mathcal{Z}} = 0 &\text{ if } z_t = \mu,\\
\\
\frac{d}{d z_t}\mathbb{E}\left \| z_{t} - Z_s \right \|_{\mathcal{Z}} > 0 &\text{ if } z_t > \mu.
    \end{cases}
\end{align}

Combining the result in Eq.~\ref{eq:proof:pdf} and Eq.~\ref{eq:proof:norm}, we have $p(z_t;\alpha)$ is maximized when $\mathbb{E}\left \| z_{t} - Z_s \right \|_{\mathcal{Z}}$ is minimized at the median $\mu$, $p(z_t;\alpha)$ increase when $\mathbb{E}\left \| z_{t} - Z_s \right \|_{\mathcal{Z}}$ decrease and vice versa. As a consequence, we obtain $p(z_t;\alpha)$ is monotonically decreasing w.r.t. distance function $\mathbb{E}\left \| z_{t} - Z_s \right \|_{\mathcal{Z}}$.

\textit{\underline{Part (2). The maximum of metric $u(x_t)$}}: Consider $u(x_t)=v(d(x_t,X_s))$ in Def.~\ref{def:distanceaware}, let $u(x_t)$ is the entropy of predictive distribution of Density-Softmax $\sigma(p(z=f(x);\alpha) \cdot (g \circ f(x)))$, then we have
\begin{align}
     u(x_t) &= \mathrm{H}(\sigma(p(z_t;\alpha) \cdot g(z_t)))\\
     &= -\sum_{i=1}^K p(y=i|\sigma(p(z_t;\alpha) \cdot g(z_t))) \log\left( p(y=i|\sigma(p(z_t;\alpha) \cdot g(z_t)))\right)\\
     &= -\sum_{i=1}^K \frac{\exp(p(z_t;\alpha) \cdot (z_{t}^\top \theta_{g_i}))}{\sum_{j=1}^K \exp(p(z_t;\alpha) \cdot (z_{t}^\top \theta_{g_j}))} \log \left ( \frac{\exp(p(z_t;\alpha) \cdot (z_{t}^\top \theta_{g_i}))}{\sum_{j=1}^K \exp(p(z_t;\alpha) \cdot (z_{t}^\top \theta_{g_j}))} \right ).
\end{align}

Let $a_i =  \frac{\exp(p(z_t;\alpha) \cdot (z_{t}^\top \theta_{g_i}))}{\sum_{j=1}^K \exp(p(z_t;\alpha) \cdot (z_{t}^\top \theta_{g_j}))}$, then we need to find
\begin{align}
    (a_1,\ldots,a_K) \text{ to maximize } -\sum_{i=1}^K(a_i) \log(a_i) \text{ subject to } \sum_{i=1}^K(a_i) - 1 = 0.
\end{align}

Since $-\sum_{i=1}^K(a_i) \log(a_i)$ is an entropy function, it strictly concave on $\Vec{a}$. In addition, because the constraint is $\sum_{i=1}^K(a_i) - 1 = 0$, the Mangasarian-Fromovitz constraint qualification holds. So, apply the Lagrange multiplier, we have the Lagrange function
\begin{align}
    \mathcal{L}(a_1,\ldots,a_K,\lambda)=  -\sum_{i=1}^K(a_i) \log(a_i) - \lambda \left ( \sum_{i=1}^K(a_i) - 1 \right ).
\end{align}

Calculate the gradient, and we obtain
\begin{align}
    \nabla_{a_1,\ldots,a_K,\lambda}  \mathcal{L}(a_1,\ldots,a_K,\lambda) &= \left ( \frac{\partial \mathcal{L}}{\partial a_1}, \ldots, \frac{\partial \mathcal{L}}{\partial a_K}, \frac{\partial \mathcal{L}}{\partial \lambda} \right )\\
    &= \left ( -(\log(a_1) + \frac{1}{\ln}) - \lambda, \ldots, -(\log(a_k) + \frac{1}{\ln}) - \lambda, \sum_{i=1}^K(a_i) -1 \right ),
\end{align}
and therefore
\begin{align}
     \nabla_{a_1,\ldots,a_K,\lambda}  \mathcal{L}(a_1,\ldots,a_K,\lambda) = 0 \Leftrightarrow 
     \begin{cases} 
-(\log(a_i)+\frac{1}{\ln}) - \lambda = 0, \forall i \in \mathcal{Y},\\ 
\\ 
\sum_{i=1}^K(a_i) -1 = 0.
    \end{cases}
\end{align}

Consider $-(\log(a_i)+\frac{1}{\ln}) - \lambda = 0, \forall i \in \mathcal{Y}$, this shows that all $a_i$ are equal (because they depend on $\lambda$ only). By using $\sum_{i=1}^K(a_i) -1 = 0$, we find 
\begin{align}
    a_i = \frac{\exp(p(z_t;\alpha) \cdot (z_{t}^\top \theta_{g_i}))}{\sum_{j=1}^K \exp(p(z_t;\alpha) \cdot (z_{t}^\top \theta_{g_j}))} = \frac{1}{K}, \forall i \in \mathcal{Y}.
\end{align}

As a consequence, $u(x_t)$ is maximized if the predictive distribution $\sigma(p(z=f(x);\alpha) \cdot (g \circ f(x))) = \mathbb{U}$, i.e., $p(z_t;\alpha) \rightarrow 0$ which will happen if $z_t$ is \acrshort{OOD} data (by the result in Thm.~\ref{theo:uniform_ood_normal_iid} and Proof~\ref{proof:theo:uniform_ood_normal_iid}).

\textit{\underline{Part (3). The monotonically decrease of metric $u(x_t)$ on the interval $\left ( 0, 1\right ]$}}: Consider the function
\begin{align}
    \mathcal{F}(p(z_t;\alpha)) = -\sum_{i=1}^K \frac{\exp(p(z_t;\alpha) \cdot (z_{t}^\top \theta_{g_i}))}{\sum_{j=1}^K \exp(p(z_t;\alpha) \cdot (z_{t}^\top \theta_{g_j}))} \log \left ( \frac{\exp(p(z_t;\alpha) \cdot (z_{t}^\top \theta_{g_i}))}{\sum_{j=1}^K \exp(p(z_t;\alpha) \cdot (z_{t}^\top \theta_{g_j}))} \right ).
\end{align}

Let $a= p(z_t;\alpha)$, $b_i = z_{t}^\top \theta_{g_i}, \forall i \in \mathcal{Y}$ then  $\mathcal{F}(a) = -\sum_{i=1}^K \frac{e^{a b_i}}{\sum_{j=1}^K e^{a b_j}} \log \left ( \frac{e^{a b_i}}{\sum_{j=1}^K e^{a b_j}} \right )$, and we need to find $\frac{d}{d a}\mathcal{F}$. Take derivative, we obtain
\begin{align}
    \frac{d}{d a}\mathcal{F} &= -\sum_{i=1}^K \left\{ \left( \frac{e^{a b_i}}{\sum_{j=1}^K e^{a b_j}}\right)' \log \left ( \frac{e^{a b_i}}{\sum_{j=1}^K e^{a b_j}} \right ) + \frac{e^{a b_i}}{\sum_{j=1}^K e^{a b_j}} \left [\log \left ( \frac{e^{a b_i}}{\sum_{j=1}^K e^{a b_j}} \right ) \right]'\right\}\\
    &= -\sum_{i=1}^K \left[ \left( \frac{e^{a b_i}}{\sum_{j=1}^K e^{a b_j}}\right)' \log \left ( \frac{e^{a b_i}}{\sum_{j=1}^K e^{a b_j}} \right ) + \frac{e^{a b_i}}{\sum_{j=1}^K e^{a b_j}} \left ( \frac{e^{a b_i}}{\sum_{j=1}^K e^{a b_j}} \right )'  \frac{\sum_{j=1}^K e^{a b_j}}{e^{a b_i}}\right]\\
    &= -\sum_{i=1}^K \left\{ \frac{b_i e^{a b_i} \sum_{j=1}^K e^{a b_j} - e^{a b_i} \sum_{j=1}^K b_j e^{a b_j}}{(\sum_{j=1}^K e^{a b_j})^2} \left[\log \left ( \frac{e^{a b_i}}{\sum_{j=1}^K e^{a b_j}} \right ) + 1 \right]\right\}\\
    &= -\sum_{i=1}^K \left\{ \frac{\sum_{j=1}^K e^{a(b_i + b_j)} (b_i - b_j)}{(\sum_{j=1}^K e^{a b_j})^2} \left[\log \left ( \frac{e^{a b_i}}{\sum_{j=1}^K e^{a b_j}} \right ) + 1 \right]\right\}.
\end{align}

By assuming the non-uniform in the predictive distribution $\sigma(g(f(x_{ood}))) \neq \mathbb{U}$ and $a \in \left(0,1 \right]$, then
\begin{align}
    \frac{d}{d a}\mathcal{F} = -\sum_{i=1}^K \left\{ \frac{\sum_{j=1}^K e^{a(b_i + b_j)} (b_i - b_j)}{(\sum_{j=1}^K e^{a b_j})^2} \left[\log \left ( \frac{e^{a b_i}}{\sum_{j=1}^K e^{a b_j}} \right ) + 1 \right]\right\} < 0,
\end{align}
combining with $u(x_t)$ is maximized if $a \rightarrow 0$, we obtain $u(x_t)$ decrease monotonically on the interval $\left(0,1 \right]$.

Combining the result in \textit{Part (2).} $u(x_t)$ is maximized if $p(z_t;\alpha) \rightarrow 0$ which will happen if $z_t$ is \acrshort{OOD} data, and the result in \textit{Part (3).} $u(x_t)$ is decrease monotonically w.r.t. $p(z_t;\alpha)$ on the interval $\left ( 0, 1\right ]$ which will happen if $x_t$ is closer to \acrshort{IID} data since the likelihood value $p(z_t;\alpha)$ increases, we obtain the distance awareness of $p(z=f(x);\alpha) \cdot g$. 
 
Combining the result in \textit{Part (1).} $p(z_t;\alpha)$ is monotonically decreasing w.r.t. distance function $\mathbb{E}\left \| z_{t} - Z_s \right \|_{\mathcal{Z}}$ and the result \textit{distance awareness} of $p(z=f(x);\alpha) \cdot g$, we obtain the conclusion: $\sigma(p(z=f(x);\alpha) \cdot (g \circ f(x)))$ is distance aware on latent space $\mathcal{Z}$ of Theorem~\ref{thm:distanceaware}. 
\end{proof}

\subsection{Proof of Proposition~\ref{prop:ECE}}\label{proof:prop:ECE}
\begin{proof}
Let us consider the prediction of the standard softmax $\sigma(g(f(x)))$. By definition, we have
\begin{align}
    \sigma:\quad\quad \mathbb{R}^K &\rightarrow \Delta_y\\
    g(f(x)) &\mapsto \sigma(g(f(x))).
\end{align}

Let the logit vectors of $g(f(x))$ be $u=(u_1,\cdots,u_K)\in \mathbb{R}^K$, for an arbitrary pair of classes, i.e., $\forall i,j\in \{1,\cdots,K\}$ of the logit vector $u$, assume that $u_i<u_j$. Since the predictive distribution of Desnity-Softmax is $\sigma(p(f(x);\alpha) \cdot g(f(x)))$, we have the corresponding logit vector is
\begin{align}
    p(f(x);\alpha) \cdot u=(p(f(x);\alpha) \cdot u_1,\cdots,p(f(x);\alpha) \cdot u_K)\in \mathbb{R}^K.
\end{align}

Due to $p(f(x);\alpha) \in (0,1]$, then we obtain the following relationship holds
\begin{align}
    [p(f(x);\alpha) \cdot u]_i<[p(f(x);\alpha) \cdot u]_j,
\end{align}
where $[\cdot]_i$ represents the $i$-th entry of the vector. Since the order of entries in the logit vector is unchanged between the standard softmax and Density-Softmax, we obtain
\begin{align}
    \argmax_{y \in \mathcal{Y}}[\sigma(g(f(x)))]_y = \argmax_{y \in \mathcal{Y}}[\sigma(p(f(x);\alpha) \cdot g(f(x))]_y.
\end{align}

As a consequence, Density-Softmax preserves the accuracy of the standard softmax by
\begin{align}\label{prop:ECE:acc}
    \text{acc}(B_m) &= \frac{1}{\left | B_m \right |} \sum_{i \in B_m} \mathbb{I}(\argmax_{y \in \mathcal{Y}}[\sigma(g(f(x_i)))]_y = y_i)\\
    &= \frac{1}{\left | B_m \right |} \sum_{i \in B_m} \mathbb{I}(\argmax_{y \in \mathcal{Y}}[\sigma(p(f(x_i);\alpha) \cdot g(f(x_i)))]_y = y_i),
\end{align}
where $B_m$ is the set of sample indices whose confidence falls into $\left ( \frac{m-1}{M}, \frac{m}{M} \right ]$ in $M$ bins.

On the other hand, since $p(f(x);\alpha) \in (0,1]$, we also have
\begin{align}
    \max_{y \in \mathcal{Y}}\left[\sigma(p(f(x);\alpha) \cdot g(f(x)))\right]_y \leq \max_{y \in \mathcal{Y}}[\sigma(g(f(x)))]_y,
\end{align}
and this yields
\begin{align}\label{prop:ECE:conf}
    \frac{1}{\left | B_m \right |} \sum_{i \in B_m} \max_{y \in \mathcal{Y}}[\sigma(p(f(x_i);\alpha) 
 \cdot g(f(x_i)))]_y \leq 
    \frac{1}{\left | B_m \right |} \sum_{i \in B_m} \max_{y \in \mathcal{Y}}[\sigma(g(f(x_i)))]_y.
\end{align}

Furthermore, by assuming the predictive distribution of the standard softmax layer $\sigma(g(f(x)))$ is over-confident, i.e., $\text{acc}(B_m) \leq \text{conf}(B_m), \forall B_m$, then we have
\begin{align}\label{prop:over-conf}
    0\leq \frac{1}{\left | B_m \right |} \sum_{i \in B_m} \mathbb{I}(\argmax_{y \in \mathcal{Y}}[\sigma(g(z_i))]_y = y_i) \leq \frac{1}{\left | B_m \right |} \sum_{i \in B_m} \max_{y \in \mathcal{Y}}[\sigma(g(z_i))]_y, \forall B_m.
\end{align}

Combining with the result in~\ref{prop:ECE:acc}, in~\ref{prop:ECE:conf} and in~\ref{prop:ECE:conf} together, for $N$ number of samples, we obtain
\begin{align}
    &\underbrace{\sum_{m=1}^M \frac{\left | B_m \right |}{N} \left | \frac{1}{\left | B_m \right |} \sum_{i \in B_m} \mathbb{I}(\argmax_{y \in \mathcal{Y}}[\sigma(p(z_i;\alpha) \cdot g(z_i))]_y = y_i) - \frac{1}{\left | B_m \right |} \sum_{i \in B_m} \max_{y \in \mathcal{Y}}[\sigma(p(z_i;\alpha) 
 \cdot g(z_i))]_y\right |}_{\text{ECE}(\sigma((p(f;\alpha) \cdot g)\circ f))}\\ 
    &\leq \underbrace{\sum_{m=1}^M \frac{\left | B_m \right |}{N} \left | \frac{1}{\left | B_m \right |} \sum_{i \in B_m} \mathbb{I}(\argmax_{y \in \mathcal{Y}}[\sigma(g(z_i))]_y = y_i) - \frac{1}{\left | B_m \right |} \sum_{i \in B_m} \max_{y \in \mathcal{Y}}[\sigma(g(z_i))]_y\right |}_{\text{ECE}(\sigma(g\circ f))}, \text{ where }z_i=f(x_i).
\end{align}
As a consequence, we obtain the conclusion: $\text{ECE}(\sigma((p(f;\alpha) \cdot g)\circ f) \leq \text{ECE}(\sigma(g \circ f))$ of Proposition~\ref{prop:ECE}.
\end{proof}

\section{Further discussion and Ablation study}\label{apd:discussion}
\subsection{Density estimation and likelihood value}\label{apd:discussdensity}
Recall that the output in the inference step of Algorithm~\ref{alg:algorithm} has the form
\begin{equation}
    p(y=i|x_t) = \frac{\exp(p(z_t;\alpha) \cdot (z_t^\top \theta_{g_i}))}{\sum_{j=1}^K \exp(p(z_t;\alpha) \cdot (z_t^\top \theta_{g_j}))}, \forall i \in \mathcal{Y}.
\end{equation}
Let's consider: $\exp(p(z_t;\alpha) \cdot (z_t^\top \theta_{g_i}))$, if the likelihood value of $p(z_t;\alpha)$ is too large, then the output of the exponential function may go to a very large value, leading to the computer can not store to compute. Therefore, we need to find a scaling technique to avoid this issue.

Recall that for density estimation, we use Normalizing-Flows~\citep{dinh2017density,papamakarios2021normalizing}. Instead of returning likelihood value $p(Z;\alpha)$, this estimation returns the logarithm of likelihood $\log(p(Z;\alpha))$, then we need to take exponentially to get the likelihood value by
\begin{align}\label{eq:explogdensity}
    p(Z;\alpha) = e^{\log(p(Z;\alpha))}.
\end{align}
By doing so, there will be two properties for the range of likelihood value $p(Z;\alpha)$. First, value of $p(Z;\alpha) \neq 0$ by $\log(0)$ being undefined. Second, $p(Z;\alpha)$ is always positive by the output of the exponential function is always positive. As a result, we obtain the likelihood value $p(Z;\alpha)$ is in the range of $\left(0,+\infty\right]$.

To avoid the numerical issue when $p(Z;\alpha) \rightarrow +\infty$, we use the scaling technique to scale the range $\left(0,+\infty\right]$ to $\left(0,1\right]$ by the following formula
\begin{align}
    p(Z;\alpha) = \frac{p(Z;\alpha)}{\max(p(Z_{train};\alpha))}.
\end{align}
It is also worth noticing that there also can be a case $e^{\log(p(Z;\alpha))}$ in Equation~\ref{eq:explogdensity} goes to $+\infty$ if $\log(p(Z;\alpha))$ is too large, therefore, we also need to scale it to avoid the numerical issue. The pseudo-code of our scaling algorithm and inference process is presented in Algorithm~\ref{alg:scaling}.

\begin{figure}[t]
\begin{algorithm}[H]
   \caption{Scaling likelihood}
   \label{alg:scaling}
\begin{algorithmic}
   \STATE {\bfseries Scaling Input:} Training data $D_s$, encoder $f$, trained density function $p(Z;\alpha)$.\;
   \STATE max\_trainNLL $\leftarrow 0$
   \FOR{$e=1\rightarrow$ epochs}
   \STATE Sample $D_{B}$ with a mini-batch $B$ for source data $D_s$\;
   \STATE $Z = f(X \in D_{B})$\;
   \STATE batch\_trainNLL $\leftarrow p(Z;\alpha)$
   \IF{$\max(\text{batch\_trainNLL}) > $ max\_trainNLL}
   \STATE max\_trainNLL $\leftarrow \max(\text{batch\_trainNLL})$
   \ENDIF
   \ENDFOR
   \STATE {\bfseries Scaling Inference Input:} Test sample $x_t$\;
   \STATE $z_t = f(x_t)$\;
   \STATE $p(z_t;\alpha) = \frac{p(z_t;\alpha)}{\text{max\_trainNLL}}$
\end{algorithmic}
\end{algorithm}
\vspace{-0.2in}
\end{figure}

\subsection{Distance preserving and 1-Lipschitz constraint}\label{apd:distance preserving}
To improve the uncertainty quality, \citet{Liu2020SNGP} introduce distance awareness definition on sample space $\mathcal{X}$, which is stronger than our definition on feature space $\mathcal{Z}$. However, to make the model distance aware on $\mathcal{X}$, one necessary condition is $f(x)$ must be an isometric mapping, i.e., distance preserving on latent space $\mathcal{Z}$ by satisfy the bi-Lipschitz condition~\citep{searcod_metric_2006}
\begin{equation}\label{eq:bi-lipschitz}
    L_1 \cdot \left \| x_1 - x_2 \right \|_{\mathcal{X}} \leq
    \left \| f(x_1) - f(x_2) \right \|_{\mathcal{Z}}
    \leq
    L_2 \cdot \left \| x_1 - x_2 \right \|_{\mathcal{X}},
\end{equation} 
for positive and bounded constants $0 < L1 < 1 < L2$. Although our model only guarantees distance awareness on $\mathcal{Z}$ and does not hold on $\mathcal{X}$, it still aims to achieve the upper bound of Eq.~\ref{eq:bi-lipschitz} by the 1-Lipschitz constraint, i.e., $||f(x_1)-f(x_2)||_2\leq ||x_1 - x_2||_2$. This helps $f(x)$ assure that if the sample is similar, the feature will be similar as well, providing meaningful correspondence with the semantic properties of the input data $\mathcal{X}$. Therefore, the 1-Lipschitz not only helps to improve robustness but may also support to improve uncertainty performance. The empirical evidence is confirmed in Table~\ref{tab:lamda} by the calibrated uncertainty error without 1-Lipschitz constraint ($\lambda=0$) is significantly higher than with 1-Lipschitz constraint (e.g., $\lambda=1e-4$).

\subsection{Additional ablation study: gradient-penalty and training cost}\label{apd:additional_ablation}
\begin{table}[ht!]
    \centering
    \begin{tabular}{lcccccc}
        \toprule
        \textbf{$\lambda$} & \textbf{NLL($\downarrow$)} & \textbf{Acc($\uparrow$)} & \textbf{ECE($\downarrow$)} & \textbf{cNLL($\downarrow$)} & \textbf{cAcc($\uparrow$)} & \textbf{cECE($\downarrow$)}\\
        \midrule
        0 & 0.142 & 96.1 & 0.015 & 0.79 & 77.0 & 0.086\\
        1 (w/o Eq.~\ref{method:2nd-optimization}) & 0.165 & 94.8 & 0.017 & 0.66 & 79.9 & 0.041\\
        1 & 0.162 & 95.0 & 0.015 & 0.64 & 80.0 & 0.039\\
        1e-1 & 0.159 & 95.2 & 0.014 & 0.66 & 79.3 & 0.040\\
        1e-2 & 0.146 & 95.6 & 0.009 & 0.66 & 79.7 & 0.047\\
        1e-3 & 0.144 & 95.7 & 0.012 & 0.66 & 79.9 & 0.051\\
        1e-4 & 0.137 & 96.0 & 0.010 & 0.68 & 79.2 & 0.060\\
        \bottomrule
    \end{tabular}
    \caption{Performance of Density-Softmax with different gradient-penalty coefficient $\lambda$, model is trained on CIFAR-10, tested on \acrshort{IID} CIFAR-10 and \acrshort{OOD} CIFAR-10-C.}
    \vspace{-0.1in}
    \label{tab:lamda}
\end{table}
Tab.~\ref{tab:lamda} shows the performance of Density-Softmax across different values of the gradient-penalty hyper-parameter $\lambda$. Firstly, we observe there is a trade-off between the performance on \acrshort{IID} and \acrshort{OOD} data w.r.t. $\lambda$, i.e., the Lipschitz constraint may affect model performance on \acrshort{IID} and \acrshort{OOD} data. Secondly, the fine-tuning classifier step in Eq.~\ref{method:2nd-optimization} is essential by a better performance of $\lambda=1$ than $\lambda=0$ (w/o Eq.~\ref{method:2nd-optimization}). Finally, with an appropriate $\lambda=1e-4$, Density-Softmax can balance the performance between \acrshort{IID} and \acrshort{OOD} data. Notably, this gradient-penalty not only helps to improve the accuracy but also the uncertainty by a lower \acrshort{ECE} than $\lambda=0$, confirming the hypothesis that the 1-Lipschitz constraint supports improving uncertainty quality at the same time.

Despite showing success in test-time efficiency, we raise awareness about the challenge of training Density-Softmax on low computational infrastructure. Specifically, Density-Softmax requires a longer-time and higher-memory cost than Deterministic~\acrshort{ERM} at training-time (e.g., Fig.~\ref{fig:train_test_barchart}), due to 3 separate training steps and the Jacobian matrix in the regularization at the first step of Algorithm~\ref{alg:algorithm}. This 1-Lipschitz regularization also requires pre-defining the hyper-parameter $\lambda$ before training in implementation (e.g., Tab.~\ref{tab:lamda}). Additionally, the 3 training steps also require pre-defining additional optimizers, iterations, and learning rates in Tab.~\ref{tab:hyper-params}. That said, our training cost is still lower than other \acrshort{SOTA} techniques, e.g., Deep Ensembles, BatchEnsemble, \acrshort{MIMO}, Rank-1~\acrshort{BNN} in Fig.~\ref{fig:train_test_barchart} (left). Importantly, the efficient benefits of Density-Softmax are illustrated with a much lower inference cost than other \acrshort{SOTA} methods at test-time in Fig.~\ref{fig:train_test_barchart} (right).
\begin{figure}[ht!]
\begin{center}
  \includegraphics[width=1.0\linewidth]{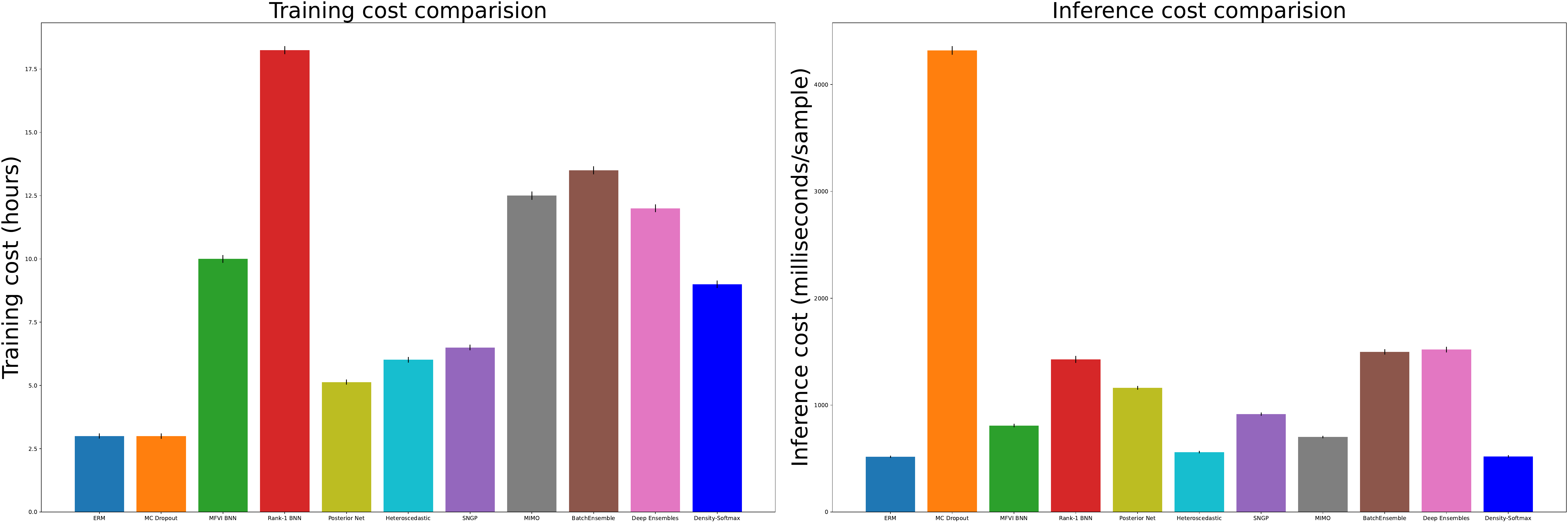}
\end{center}
\vspace{-0.1in}
\caption{Comparison in training cost (hours) on NVIDIA~A100 and inference cost (milliseconds/sample) on NVIDIA RTX~A5000, with error bars across 10 seeds of Wide~Resnet-28-10 on CIFAR-10 dataset. Density-Softmax outperforms \acrshort{SOTA} in inference speed and still has a lower training cost than some baselines.}
\vspace{-0.16in}
\label{fig:train_test_barchart}
\end{figure}

\subsection{Additional discussion about our method}\label{apd:additional_discuss}
\textbf{Lipschitz gradient penalty and the sufficient condition to 1-Lipschitz.} Since the feature extractor $f$ is a neural network, we can not verify whether $f$ is a 1-Lipschitz. However, there are two observations we want to discuss. Firstly, we note that our feature extractor $f$ satisfies $\sup_x||\nabla_x f(x)||_2=1$ in training data with the Rademacher Theorem~\ref{theo:radamacher}. Secondly, we tried the SeqLip~\citep{virmaux2018Lipschitz} to approximate the Lipschitz constant $L(f)$ by computing the upper bound of $L(f)$ with the Wide Renet 28-10 backbone on the CIFAR-10. We observe that the upper bounds of $L(f)$ are about $4.5\cdot 10^6$ without gradient-penalty and $2.8\cdot 10^6$ with gradient-penalty regularization. This does not show whether our true $L(f)$ is equal to 1 or not, but at least, it shows the gradient-penalty regularization can reduce the upper bound of $L(f)$.

\textbf{Potential to use with pre-trained large models.} Although we train the backbone from scratch, our method is flexible to be used to improve the uncertainty of a pre-trained model. Specifically, our training of the density function in Eq.~\ref{method:densityestimate} and our update classifier weight in Eq.~\ref{method:2nd-optimization} do not require backpropagation through the feature backbone $f$. Therefore, we can use these two training steps to fine-tune a pre-trained model to enhance uncertainty estimation, like Vision Transformer~\citep{dosovitskiy2021an}. We also show this improvement in Tab.~\ref{tab:lamda}, with $\lambda=0$, our method achieves $0.015$ in ECE and $0.086$ in cECE, while the pre-trained model, i.e., Deterministic~\acrshort{ERM} only achieves $0.023$ in ECE and $0.053$ in cECE in Tab.~\ref{tab:cifar10_results}.

\textbf{Comparison with \acrshort{OOD} detection related work.} Our method aims to improve calibration (both empirical and theoretical analysis of the calibration level w.r.t. \acrshort{ECE} in Prop.~\ref{prop:ECE}) but also can be used for \acrshort{OOD} detection in Sec.~\ref{experiments}, while \acrshort{OOD} detection methods~\citep{dosovitskiy2021an,lee2018asimple,sun2022OOD}, in contrast, are not explicitly designed to improve calibration. In terms of the test-time efficiency, these methods also suffer from a heavy computational burden. In particular, \citet{lee2018asimple} calculates expensive invertible covariance matrixes whose size depends on the feature dimension and the class number for multiple \acrshort{DNN} layers, additionally the cost of the derivative w.r.t. the input sample. Meanwhile, \citet{dosovitskiy2021an,sun2022OOD} are non-parametric methods whose computational complexity at test-time depends on the number of training samples. In contrast, our model estimates a marginal lightweight density function and is a parametric method, hence, the complexity at test-time does not depend on the size of the training set.

\section{Additional experiments}\label{apd:exp}
\subsection{Baseline details}\label{apd:baseline}
This appendix provides an exhaustive literature review of 14 \acrshort{SOTA} related methods which are used to make comparisons with our model by using the uncertainty-baselines~\citep{nado2021uncertainty} (We exclude Mixup~\citep{zhang2018mixup,carratino2022mixup} and Augmix~\citep{hendrycks*2020augmix} by we do not use data-augmentation):
\begin{itemize}
    \item \textbf{Deterministic~\acrshort{ERM}~\citep{vapnik1998erm}} is a standard deterministic model. In test-time, it predicts the label immediately by a single forward pass.
    \item \textbf{\acrshort{MC}~Dropout~\citep{gal2016mcdropout}} includes dropout regularization method in the model. In test-time, it uses \acrshort{MC} sampling by dropout to make different predictions, then get the mean of the list prediction.
    \item \textbf{\acrshort{MFVI}~\acrshort{BNN}~\citep{wen2018flipout}} uses the \acrshort{BNN} by putting distribution over the weight by mean and variance per each weight. Because each weight consists of mean and variance, the total model weights will double as the Deterministic~\acrshort{ERM} model.
    \item \textbf{Rank-1~\acrshort{BNN}~\citep{dusenberry2020rank1}} use the mixture approximate posterior to benefits from multimodal representation with local uncertainty around each mode, compute posterior on rank-1 matrix while keeping weight is deterministic. In inference time, sampling weight distribution by \acrshort{MC} to make different predictions, then get the mean of the list prediction.
    \item \textbf{Posterior~Net~\citep{charpentier2020postnet}} is based on Evidential Deep Learning~\citep{sensoy2018evidential} with a latent density function by utilizing conditional densities per class, intuitively acting as class conditionals on the latent space. Due to belonging to the Bayesian perspective, it needs to select a "good" Prior distribution, which is difficult in practice.
    \item \textbf{Heteroscedastic~\citep{collier2021correlated}} is a probabilistic approach to modeling input-dependent by placing a multivariate Normal distributed latent variable on the final hidden layer of a neural network classifier. In test-time, it uses \acrshort{MC} sampling to make different predictions, then get the mean of the list prediction.
    \item \textbf{\acrshort{SNGP}~\citep{Liu2020SNGP}} is a combination of the last Gaussian Process layer with 
    Spectral Normalization to the hidden layers. Because using the Spectral Normalization for every weight in each residual layer with the power iteration method~\citep{gouk2021regularisation,miyato2018spectral}, and Gaussian Process layer with \acrshort{MC} sampling, the weight and latency of \acrshort{SNGP} is considerably higher than the Deterministic~\acrshort{ERM} model. 
    \item \textbf{\acrshort{MIMO}}~\citep{havasi2021training}, i.e., multi-input multi-output, it trains multiple independent subnetworks within a network. In test-time, it performs multiple independent predictions in a single forward pass.
    \item \textbf{\acrshort{DUQ}~\citep{van2020uncertainty}} is a deterministic-based framework by using kernel distance with Radial Basis Function to improve \acrshort{OOD} detection ability. However, it has been shown often to have a poor performance in uncertainty and robustness~\citep{Liu2020SNGP,nado2021uncertainty}. Regarding scalability, this requires storing and computing a weight matrix with the size depending quadratically on latent dimension and linearly in the number of classes, leading to high computational demands across modern \acrshort{DNN} architectures at inference time.
    \item \textbf{\acrshort{DDU}~\citep{mukhoti2022deep}} is another deterministic-based method, which is a combination of \acrshort{ERM} with a Gaussian Mixture Models (GMM) density function to detect OOD samples. However, this \acrshort{OOD} detector does not contribute to the predictive distribution of the softmax directly. As a result, this method needs to apply the post-hoc re-calibration technique to improve calibration performance. Meanwhile, we focus on methods that do not require additional calibration data and post-hoc calibration steps, i.e., without Temperature Scaling (w/o TS).
    \item \textbf{\acrshort{DUE}~\citep{vanamersfoort2022feature}} is an extension version of \acrshort{SNGP} by constrain Deep Kernel Learning’s feature extractor to approximately preserve distances through a bi-Lipschitz constraint.
    \item \textbf{\acrshort{NatPN}~\citep{charpentier2022natural}} is based on Posterior~Net and is the closest to our work by also estimating the density function on the marginal feature space. Yet, we have two main differences: (1) From the Bayesian view, optimizing with \acrshort{NatPN} is equivalent to doing Maximum a Posteriori estimation while ours is equivalent to doing \acrshort{MLE}. (2) The predictive distribution in \acrshort{NatPN} is not a standard softmax output because it is not normalized by a natural exponent function. Meanwhile, Density-Softmax is normalized by a natural exponent function. So, we have the following benefits: (1) Density-Softmax can calculate the predictive distribution without the need to consider how to define prior parameters. The prior can hurt the posterior performance in \acrshort{NatPN} if we pre-define a bad prior.  (2) The natural exponential helps Density-Softmax easily optimize with the cross-entropy loss by the nice derivative property of $\ln(\exp(a)) = a$~\citep{GoodBengCour16}, and produce a sharper distribution with a higher probability of the largest entry in the logit vector and lower probabilities of the smaller entries when compared with the normalization of \acrshort{NatPN}.
    \item \textbf{BatchEnsemble~\citep{Wen2020BatchEnsemble:}} defines each weight matrix to be the Hadamard product of a shared weight among ensemble members and a rank-1 matrix per member. In test-time, the final prediction is calculated from the mean of the list prediction of the ensemble.
    \item \textbf{Deep Ensembles~\citep{lakshminarayanan2017ensemble}} includes multiple Deterministic~\acrshort{ERM} trained with different seeds. In test-time, the final prediction is calculated from the mean of the list ensemble predictions. Due to aggregates from multiple models, the total of model weights needed to store will increase linearly w.r.t. the number of models.
\end{itemize}

\subsection{Implementation details}\label{apd:implementation}
Based on the code of~\citet{nado2021uncertainty}, we use similar settings with their Deterministic~\acrshort{ERM} for everything for a fair comparison on the benchmark dataset. For the setting on the toy dataset, we follow the settings of~\citet{Liu2020SNGP}.

\textbf{Datasets.} We utilize 6 commonly used datasets under distributional shifts~\citep{nado2021uncertainty}, including Toy dataset~\citep{Liu2020SNGP} with two moons and two ovals to visualize uncertainty and clustering. CIFAR-10-C, CIFAR-100-C, and ImageNet-C~\citep{hendrycks2018benchmarking} for the main benchmarking. To evaluate real-world shifts, we experiment on SVHN~\citep{netzer2011SVHN} and CIFAR-10.1~\citep{recht2018cifar10.1}.

\textbf{Evaluation metrics.} To evaluate the generalization, we use \acrshort{NLL} and Accuracy. For uncertainty estimation, we visualize uncertainty surfaces (entropy, different predictive variances), predictive entropy, \acrshort{AUPR}/\acrshort{AUROC} for \acrshort{OOD} detection, and \acrshort{ECE} with $15$ bins for calibration. To compare the robustness under distributional shifts, we evaluate every \acrshort{OOD} set in each dataset. To compare computational efficiency, we count the number of model parameters for storage requirements and measure latency in milliseconds per sample \textbf{at test-time}.

\textbf{Training details.} We train models on the train set excluding data augmentation~\citep{hendrycks*2020augmix,zhang2018mixup}, test on the original \acrshort{IID} test, and aforementioned \acrshort{OOD} sets. The performance is evaluated on backbones ResFFN-12-128~\citep{Liu2020SNGP} for the Toy, Wide~Resnet-28-10~\citep{zagoruyko2016wideresnet} for CIFAR-10-100, and Resnet-50~\citep{he2016resnet} for ImageNet. We only use normalization to process images for the benchmark datasets. Specifically, for CIFAR-10 and CIFAR-100, we normalize by the mean is $[0.4914, 0.4822, 0.4465]$ and standard deviation is $[0.2470, 0.2435, 0.2616]$. For ImageNet, we normalize by the mean is $[0.485, 0.456, 0.406]$ and standard deviation is $[0.229, 0.224, 0.225]$ (note that we do not perform augmentation techniques of~\citet{hendrycks*2020augmix} and~\citet{zhang2018mixup}).

\textbf{Architectures and hyper-parameters.} We list the detailed value of hyper-parameters used for each dataset in Table~\ref{tab:hyper-params} and the architecture of Normalizing-Flows in Table~\ref{tab:flows}. For a fair comparison with deterministic, we always use the same hyper-parameters as the Deterministic~\acrshort{ERM}.

\begin{table}[t!]
\begin{minipage}{1.0\textwidth}
\caption{Condition dataset, hyper-parameters, and their default values in our experiments. Settings are inherited from~\citet{Liu2020SNGP} for the toy dataset and from~\citet{nado2021uncertainty} for CIFAR-10-100 and ImageNet. Note that we always use the same hyper-parameters as the Deterministic~\acrshort{ERM} for a fair comparison.}
\label{tab:hyper-params}
\centering
\scalebox{0.9}{
\begin{tabular}{ccc}
\toprule
\textbf{Conditions} & \textbf{Hyper-parameters} & \textbf{Default value} \\
\midrule
\multirow{7}{*}{Toy dataset}
  & backbone & ResFFN-12-128~\citep{Liu2020SNGP}\\
  & epochs & 100\\ 
  & batch size & 128\\
  & optimizer & Adam~\citep{adam}\\
  & learning rate & 1e-4\\ 
  & density estimation epochs (Flows) & 300\\
  & density optimizer (Flows) & Adam~\citep{adam}\\
  & density learning rate (Flows) & 1e-4\\
  & re-optimize classifier epochs & 1\\
  & gradient-penalty coefficient & 1e-2\\
\midrule
\multirow{13}{*}{CIFAR-10-100}
  & backbone & Wide~Resnet-28-10~\citep{zagoruyko2016wideresnet}\\
  & epochs & 200\\ 
  & checkpoint interval & 25\\
  & batch size & 64\\
  & optimizer & SGD(momentum = 0.9, nesterov = True)\\
  & learning rate & 0.05\\ 
  & lr decay epochs & [60, 120, 160]\\
  & lr decay ratio & 0.2\\
  & L2 regularization coefficient & 2e-4\\
  & scale range & -[0.55-0.4, 0.4-0.3]\\
  & density estimation epochs & 50\\
  & density optimizer & Adam~\citep{adam}\\
  & density learning rate & 1e-4\\
  & re-optimize classifier epochs & 10\\
  & gradient-penalty coefficient & 1e-4, 1e-5 (CIFAR-10, CIFAR-100)\\
\midrule
\multirow{13}{*}{ImageNet} 
  & backbone & Resnet-50~\citep{he2016resnet}\\
  & epochs & 90\\ 
  & checkpoint interval & 25\\
  & batch size & 64\\
  & optimizer & SGD(momentum = 0.9, nesterov = True)\\
  & learning rate & 0.05\\ 
  & lr decay epochs & [30, 60, 80]\\
  & lr decay ratio & 0.1\\
  & L2 regularization coefficient & 1e-4\\
  & scale range & -[0.4, 0.3]\\
  & density estimation epochs & 10\\
  & density optimizer & Adam~\citep{adam}\\
  & density learning rate & 1e-4\\
  & re-optimize classifier epochs & 2\\
  & gradient-penalty coefficient & 1e-1\\
\bottomrule
\end{tabular}}
\end{minipage}
\begin{minipage}{1.0\textwidth}
\vspace{0.2in}
\caption{Two Coupling structures-(a) and (b) for Normalizing-Flows~\citep{dinh2017density,papamakarios2021normalizing} on the Toy, CIFAR-10, CIFAR-100, and ImageNet. Input: $Z \in \mathbb{R}^{d_z}$, where $d_z = 128$ for ResFFN-12-128, $d_z = 640$ for Wide~Resnet-28-10, and $d_z=2048$ for Resnet-50.}
\label{tab:flows}
\scalebox{0.9}{
\begin{tabular}{cc}
    \begin{minipage}{0.55\linewidth}
    \centering
        \begin{tabular}{c}
        \toprule
            (a) $S_{\theta_S} $\\
            \midrule
            Input: $Z \in \mathbb{R}^{d_z}$\\
            \midrule
            Dense (units = 16, regularizers l2 = 0.01, ReLU)\\
            \midrule
            Dense (units = 16, regularizers l2 = 0.01, ReLU)\\
            \midrule
            Dense (units = 16, regularizers l2 = 0.01, ReLU)\\
            \midrule
            Dense (units = 16, regularizers l2 = 0.01, ReLU)\\
            \midrule
            Dense (units = 16, regularizers l2 = 0.01, Tanh)\\
            \bottomrule
        \end{tabular}
    \end{minipage} 
    &
    \begin{minipage}{0.55\linewidth}
    \centering
        \begin{tabular}{c}
        \toprule
            (b) $T_{\theta_T} $\\
           \midrule
            Input: $Z \in \mathbb{R}^{d_z}$\\
            \midrule
            Dense (units = 16, regularizers l2 = 0.01, ReLU)\\
            \midrule
            Dense (units = 16, regularizers l2 = 0.01, ReLU)\\
            \midrule
            Dense (units = 16, regularizers l2 = 0.01, ReLU)\\
            \midrule
            Dense (units = 16, regularizers l2 = 0.01, ReLU)\\
            \midrule
            Dense (units = 16, regularizers l2 = 0.01, Linear)\\
            \bottomrule
        \end{tabular}
    \end{minipage} 
\end{tabular}}
\end{minipage}
\end{table}

\textbf{Source code and computing system.}
Our source code includes the notebook demo on the toy dataset, scripts to download the benchmark dataset, setup for environment configuration, and our provided code (detail in README.md). We train our model on two single GPUs: NVIDIA A100-PCIE-40GB with 8 CPUs: Intel(R) Xeon(R) Gold 6248R CPU @ 3.00GHz with 8GB RAM per each, and require 100GB available disk space for storage. We test our model on three different settings, including (1) a single GPU: NVIDIA Tesla~K80 accelerator-12GB GDDR5 VRAM with 8-CPUs: Intel(R) Xeon(R) Gold 6248R CPU @ 3.00GHz with 8GB RAM per each; (2) a single GPU: NVIDIA RTX~A5000-24564MiB with 8-CPUs: AMD Ryzen Threadripper 3960X 24-Core with 8GB RAM per each; and (3) a single GPU: NVIDIA A100-PCIE-40GB with 8 CPUs: Intel(R) Xeon(R) Gold 6248R CPU @ 3.00GHz with 8GB RAM per each.

\textbf{Demo notebook code for Algorithm~\ref{alg:algorithm}}\label{apd:code}

\begin{tabular}{cc}
\begin{minipage}{0.1\linewidth}
\end{minipage}&
\begin{minipage}{0.96\linewidth}
\begin{minted}
[
frame=lines,
framesep=2mm,
baselinestretch=1,
bgcolor=LightGray,
fontsize=\fontsize{8.2pt}{8.2pt},
linenos
]
{python}
import tensorflow as tf

#Define a features extractor f.
encoder = tf.keras.Sequential([
	tf.keras.layers.Dense(100, activation = "relu"),
	tf.keras.layers.Dense(100, activation = "relu"),
])
#Define a classifier g.
classifier = tf.keras.layers.Dense(10)

#Define a tf step function to pre-train model w.r.t. Eq. 4.
@tf.function
def pre_train_step(x, y):
    with tf.GradientTape(persistent = True) as tape:
        tape.watch(x)
        features = encoder(x)
        logits = classifier(features)
        loss_value = tf.cast(loss_fn(y, logits), tf.float64)
        grad_norm = tf.sqrt(tf.reduce_sum(tape.batch_jacobian(features, x)**2, axis = [1, 2]))
        loss_value += (grad_lambda * tf.reduce_mean(((grad_norm - 1)**2)))
        
    list_weights = encoder.trainable_weights + classifier.trainable_weights
    grads = tape.gradient(loss_value, list_weights)
    optimizer.apply_gradients(zip(grads, list_weights))
    return loss_value

#Define a tf step function to re-update the classifier by feature density model w.r.t. Eq. 6.
@tf.function
def train_step(x, y, flow_model, train_likelihood_max):
    with tf.GradientTape() as tape:
        features = encoder(x)
        likelihood = tf.exp(flow_model.score_samples(features))
        likelihood = (tf.expand_dims(likelihood, 1))/(train_likelihood_max)
        logits = classifier(features) * tf.cast(likelihood, dtype=tf.float32)
        loss_value = loss_fn(y, logits)

    grads = tape.gradient(loss_value, classifier.trainable_weights)
    optimizer.apply_gradients(zip(grads, classifier.trainable_weights))
    return loss_value

#Define a tf step function to make inference w.r.t. Eq. 7.
@tf.function
def test_step(x, flow_model, train_likelihood_max):
    features = encoder(x)
    likelihood = tf.exp(flow_model.score_samples(features))
    likelihood = (tf.expand_dims(likelihood, 1))/(train_likelihood_max)
    logits = classifier(features) * tf.cast(likelihood, dtype=tf.float32)
    y_pred = tf.nn.softmax(logits)
    return y_pred
\end{minted}
\end{minipage}
\end{tabular}

\subsection{Empirical result details}\label{apd:results}
\subsubsection{Distance awareness on the toy dataset}\label{apd:results_toy}
\begin{figure}[ht!]
\begin{center}
  \includegraphics[width=1.0\linewidth]{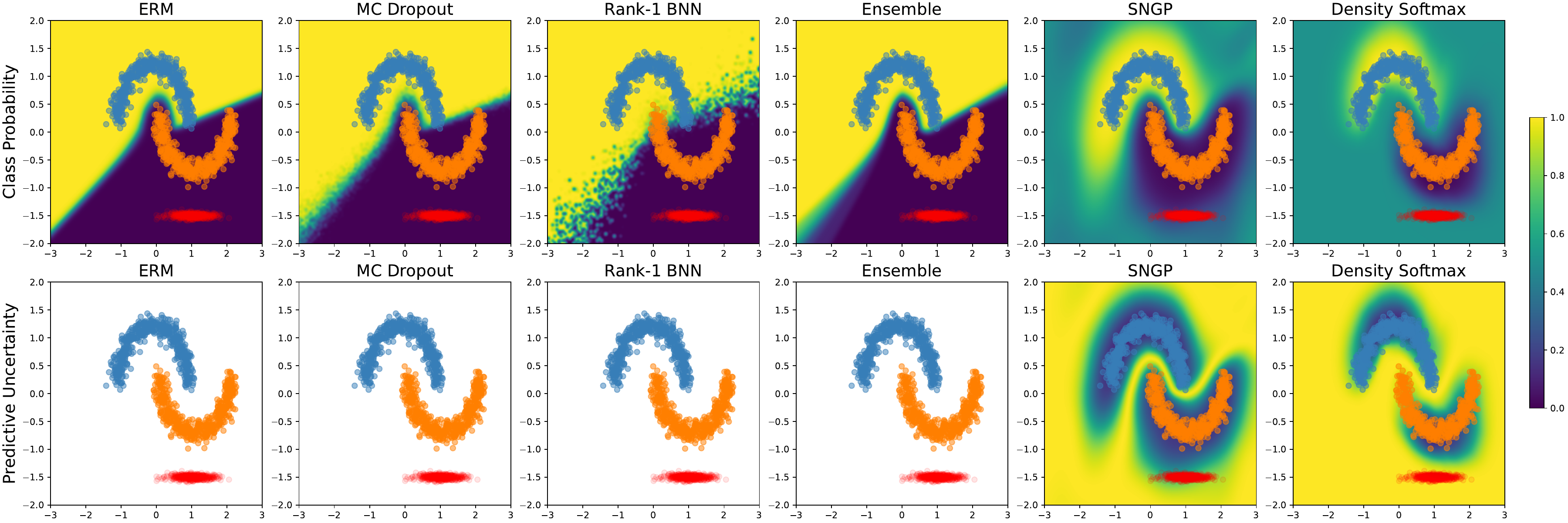}
\end{center}
    \vspace{-0.2in}
  \caption{The class probability $p(y|x)$ (Top Row) and predictive entropy $H(Y|X=x) = -p(y|x) \cdot \log_2(p(y|x)) - (1-p(y|x)) \cdot \log_2((1-p(y|x)))$ surface (Bottom Row) between different approaches. Density-Softmax achieves distance awareness with a uniform class probability and high entropy value on \textcolor{red}{OOD data}. Note that the white background due to $\log(0)$ is undefined.}
  \vspace{-0.1in}
\label{fig:2dmoon-entropy}
\end{figure}
\textbf{Different uncertainty metrics.} Figure~\ref{fig:2dmoon-entropy} presents a comparison in terms of predictive entropy. Because of the over-confidence of Deterministic~\acrshort{ERM}, \acrshort{MC}~Dropout, Rank-1~\acrshort{BNN}, and Ensemble which return $p(y|x) = 0$, theirs predictive entropy become undefined by $\log_2(p(y|x)) = \log_2(0)$ is undefined. As a result, we have a white background for these mentioned baselines, showing these methods are not able to be aware of the distance \acrshort{OOD} dataset. For \acrshort{MC}~Dropout and Ensemble, since
these two methods do not provide a convenient expression of the predictive variance of Bernoulli distribution, we also follow the work of~\citet{Liu2020SNGP} to plot their predictive uncertainty as the distance of the maximum predictive probability from $0.5$ so that $u(x) \in \left[0,1\right]$ in Figure~\ref{fig:2dmoon-var2}. Although with different uncertainty metrics, Deterministic~\acrshort{ERM}, \acrshort{MC}~Dropout, Rank-1~\acrshort{BNN}, and Ensemble still perform badly, showing truly unable to have the distance-aware ability.

\begin{figure}[ht!]
\begin{center}
  \includegraphics[width=1.0\linewidth]{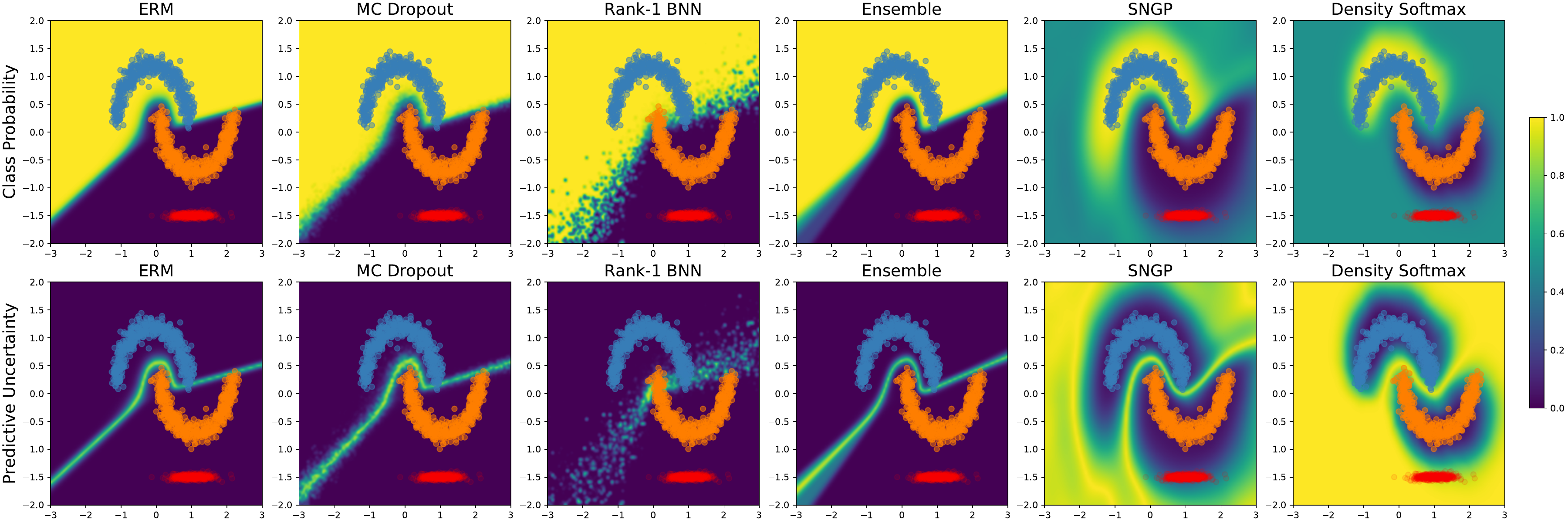}
\end{center}
    \vspace{-0.2in}
  \caption{The class probability $p(y|x)$ and predictive uncertainty $u(x) = 1 - 2 \cdot \left | p(y|x) - 0.5 \right |$ surface between different approaches.}
  \vspace{-0.2in}
\label{fig:2dmoon-var2}
\end{figure}

\textbf{Different toy dataset.} To continue to verify a consistency observation across different empirical settings, we illustrate distance-aware ability by predictive variance surface on the second toy dataset. Figure~\ref{fig:2doval} compares the methods as mentioned earlier on two ovals datasets. Similar to previous settings, only \acrshort{SNGP} can show its distance-aware uncertainty. However, for around half of \acrshort{OOD} data points, its predictive uncertainty is still $0.0$, showing still has limitations in distance-aware ability. In contrast, our Density-Softmax performs well, with $0.0$ uncertainty value on \acrshort{IID} training data points while always returning $1.0$ on \acrshort{OOD} data points, confirming the hypothesis that our Density-Softmax is a better distance-aware method. 

\begin{figure}[ht!]
\vspace{-0.1in}
\begin{center}
  \includegraphics[width=1.0\linewidth]{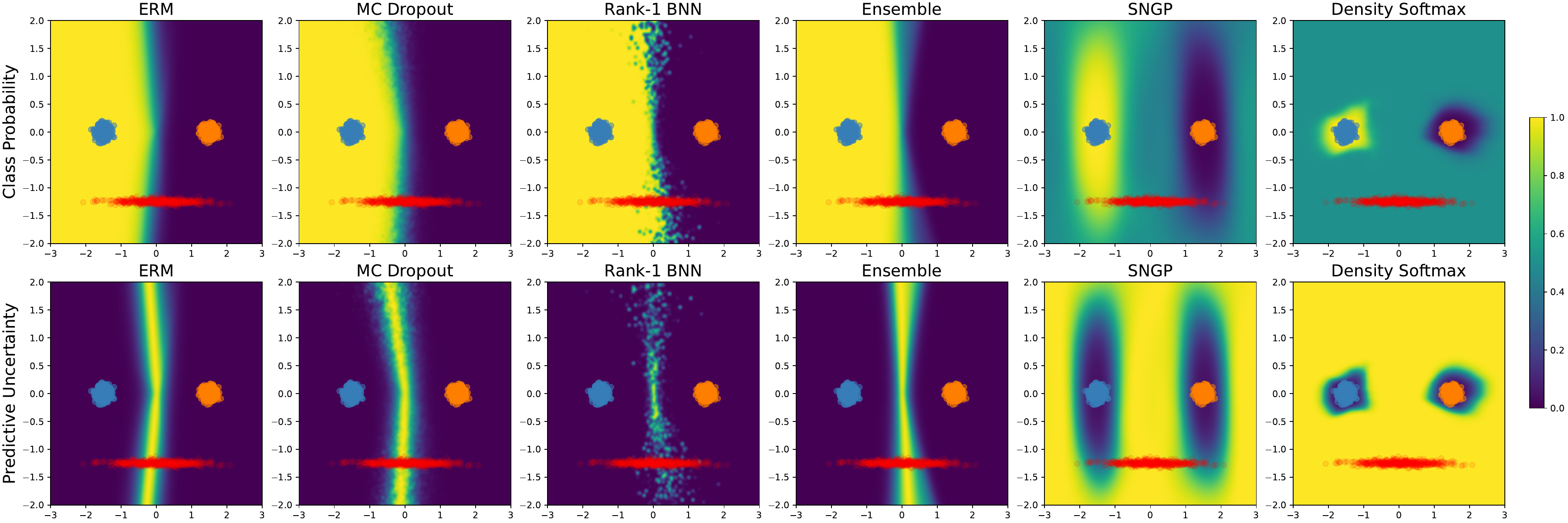}
\end{center}
    \vspace{-0.2in}
  \caption{The class probability $p(y|x)$ and predictive uncertainty $u(x) = 1 - 2 \cdot \left | p(y|x) - 0.5 \right |$ surface between different approaches on the two ovals 2D classification dataset.}
\label{fig:2doval}
\vspace{-0.4in}
\end{figure}

\subsubsection{Benchmark result details}\label{apd:results_benchmark}
\begin{table*}[t!]
\begin{minipage}{1.0\textwidth}
\caption{Detailed table of Tab.~\ref{tab:cifar10_results}.}
\label{tab:cifar10_results+}
\centering
\scalebox{0.78}{
\begin{tabular}{lccccccccccc}
\toprule
\textbf{Method} & \textbf{NLL($\downarrow$)} & \textbf{Acc($\uparrow$)} & \textbf{ECE($\downarrow$)} & \textbf{cNLL($\downarrow$)} & \textbf{cAcc($\uparrow$)} & \textbf{cECE($\downarrow$)} & \textbf{oNLL($\downarrow$)} & \textbf{oAcc($\uparrow$)} & \textbf{oECE($\downarrow$)} & \textbf{\#Params($\downarrow$)} & \textbf{Latency($\downarrow$)}\\
\midrule
Deterministic~ERM & 0.159 & 96.0 & 0.023 & 1.05 & 76.1 & 0.153 & 0.40 & 89.9 & 0.064 & \textbf{36.50M} & \textbf{518.12}\\
MC Dropout & 0.145 & 96.1 & 0.019 & 1.27 & 70.0 & 0.167 & 0.39 & 89.9 & 0.058 & 36.50M & 4,319.01\\
MFVI BNN & 0.211 & 94.7 & 0.029 & 1.46 & 71.3 & 0.181 & 0.49 & 88.1 & 0.070 & 72.96M & 809.01\\
Rank-1 BNN & 0.128 & 96.3 & \textbf{0.008} & 0.84 & 76.7 & 0.080 & 0.32 & 90.4 & 0.033 & 36.65M & 1,427.67\\
Posterior~Net & 0.360  & 93.1  & 0.112 & 1.06 & 75.2 & 0.139 & 0.42 & 87.9 & 0.053 & 36.60M & 1,162.48 \\
Heteroscedastic & 0.156 & 96.0 & 0.023 & 1.05 & 76.1 & 0.154 & 0.38 & 90.1 & 0.056 & 36.54M & 560.43\\
SNGP & 0.134 & 96.0 & 0.007 & 0.74 & 78.5 & 0.078 & 0.43 & 89.7 & 0.064 & 37.50M & 916.26\\
MIMO & 0.123 & 96.4 & 0.010 & 0.93 & 76.6 & 0.112 & 0.35 & 90.1 & 0.037 & 36.51M & 701.66\\
DUQ & 0.239 & 94.7 & 0.034 & 1.35 & 71.6 & 0.183 & 0.49 & 87.9 & 0.068 & 40.61M & 1,013.50\\
DDU (w/o TS) & 0.159 & 96.0 & 0.024 & 1.06 & 76.0 & 0.153 & 0.39 & 89.8 & 0.063 & 37.60M & 954.31\\
DUE & 0.145 & 95.6 & 0.007 & 0.84 & 77.8 & 0.079 & 0.46 & 89.2 & 0.066 & 37.50M & 916.26\\
NatPN & 0.242  & 92.8  & 0.041 & 0.89 & 73.9 & 0.121 & 0.46 & 86.3 & 0.049 & 36.58M & 601.35\\
BatchEnsemble & 0.136 & 96.3 & 0.018 & 0.97 & 77.8 & 0.124 & 0.35 & 90.6 & 0.048 & 36.58M & 1,498.01\\
Deep Ensembles & \textbf{0.114} & \textbf{96.6} & 0.010 & 0.81 & 77.9 & 0.087 & 0.28 & \textbf{92.2} & 0.025 & 145.99M &  1,520.34\\
\textbf{Density-Softmax} & 0.137 & 96.0 & 0.010 & \textbf{0.68} & \textbf{79.2} & \textbf{0.060} & \textbf{0.26} & 91.6 & \textbf{0.016} &  \color{blue}{36.58M} & \color{blue}{\textbf{520.53}}\\
\bottomrule
\end{tabular}}
\end{minipage}

\begin{minipage}{1.0\textwidth}
\caption{Detailed table of Tab.~\ref{tab:cifar100_results}.} 
\label{tab:cifar100_results+}
\centering
\scalebox{0.8}{
\begin{tabular}{lcccccccccc}
\toprule
\textbf{Method} & \textbf{NLL($\downarrow$)} & \textbf{Acc($\uparrow$)} & \textbf{ECE($\downarrow$)} & \textbf{cNLL($\downarrow$)} & \textbf{cAcc($\uparrow$)} & \textbf{cECE($\downarrow$)} & \textbf{AUPR-S($\uparrow$)} & \textbf{AUPR-C($\uparrow$)} & \textbf{\#Params($\downarrow$)} & \textbf{Latency($\downarrow$)}\\
\midrule
Deterministic~ERM & 0.875 & 79.8 & 0.086 & 2.70 & 51.4 & 0.239 & 0.882 & 0.745 & \textbf{36.55M} & \textbf{521.15}\\
MC Dropout & 0.785 & 80.7 & 0.049 & 2.73 & 46.2 & 0.207 & 0.832 & 0.757 & 36.55M & 4,339.03\\
MFVI BNN & 0.944 & 77.8 & 0.097 & 3.18 & 48.2 & 0.271 & 0.882 & 0.748 & 73.07M & 818.31\\
Rank-1 BNN  & 0.692 & 81.3 & \textbf{0.018} & 2.24 & 53.8 & 0.117 & 0.884 & 0.797 & 36.71M & 1,448.90\\
Posterior~Net & 2.021 & 77.3 & 0.391 & 3.12 & 48.3 & 0.281 & 0.880 & 0.760 & 36.60M & 11,243.84\\
Heteroscedastic & 0.833 & 80.2 & 0.059 & 2.40 & 52.1 & 0.177 & 0.881 & 0.752 & 37.00M & 568.17\\
SNGP & 0.805 & 80.2 & 0.020 & 2.02 & 54.6 & 0.092 & \textbf{0.923} & 0.801 & 37.50M & 926.99\\
MIMO & 0.690 & 82.0 & 0.022 & 2.28 & 53.7 & 0.129 & 0.885 & 0.760 & 36.68M & 718.11\\
DUQ & 0.980 & 78.5 & 0.119 & 2.84 & 50.4 & 0.281 & 0.878 & 0.732 & 77.58M & 1,034.66\\
DDU (w/o TS) & 0.877 & 79.7 & 0.086 & 2.70 & 51.3 & 0.240 & 0.890 & 0.797 & 37.60M & 959.25\\
DUE & 0.902 & 79.1 & 0.038 & 2.32 & 53.5 & 0.127 & 0.897 & 0.787 & 37.50M & 926.99\\
NatPN & 1.249 & 76.9 & 0.091 & 2.97 & 48.0 & 0.265 & 0.875 & 0.768 & 36.64M & 613.44\\
BatchEnsemble & 0.690 & 81.9 & 0.027 & 2.56 & 53.1 & 0.149 & 0.870 & 0.757 & 36.63M & 1,568.77\\
Deep Ensembles & \textbf{0.666} & \textbf{82.7} & 0.021 & 2.27 & 54.1 & 0.138 & 0.888 & 0.780 & 146.22M & 1,569.23\\
\textbf{Density-Softmax} & 0.780 & 80.8 & 0.038 & \textbf{1.96} & \textbf{54.7} & \textbf{0.089} & 0.910 & \textbf{0.804} & \color{blue}{36.64M} & \color{blue}{\textbf{522.94}}\\
\bottomrule
\end{tabular}}
\end{minipage}

\begin{minipage}{1.0\textwidth}
\caption{Detailed table of Tab.~\ref{tab:imagenet_results}.}
\label{tab:imagenet_results+}
\centering
\scalebox{0.8}{
\begin{tabular}{lcccccccc}
\toprule
\textbf{Method} & \textbf{NLL($\downarrow$)} & \textbf{Acc($\uparrow$)} & \textbf{ECE($\downarrow$)} & \textbf{cNLL($\downarrow$)} & \textbf{cAcc($\uparrow$)} & \textbf{cECE($\downarrow$)} & \textbf{\#Params($\downarrow$)} & \textbf{Latency($\downarrow$)}\\
\midrule
Deterministic~ERM & 0.939 & 76.2 & 0.032 & 3.21 & 40.5 & 0.103 & \textbf{25.61M} & \textbf{299.81}\\
MC Dropout & 0.919 & 76.6 & 0.026 & 2.96 & 42.4 & 0.046 & 25.61M & 601.15\\
Rank-1 BNN & 0.886 & 77.3 & 0.017 & 2.95 & 42.9 & 0.054 & 26.35M & 690.14\\
Heteroscedastic & 0.898 & 77.5 & 0.033 & 3.20 & 42.4 & 0.111 & 58.39M & 337.50\\
SNGP & 0.931  & 76.1  & \textbf{0.013} & 3.03 & 41.1 & 0.045 & 26.60M & 606.11\\
MIMO & 0.887  & 77.5  & 0.037 & 3.03 & 43.3 & 0.106 & 27.67M & 367.17\\
BatchEnsemble & 0.922 & 76.8 & 0.037 & 3.09 & 41.9 & 0.089 & 25.82M & 696.81\\
Deep Ensembles & \textbf{0.857} & \textbf{77.9} & 0.017 & 2.82 & \textbf{44.9} & 0.047 & 102.44M& 701.34\\
\textbf{Density-Softmax} & 0.885 & 77.5 & 0.019  & \textbf{2.81} & 44.6 & \textbf{0.042} & \color{blue}{25.88M} & \color{blue}{\textbf{299.90}}\\
\bottomrule
\end{tabular}}
\end{minipage}
\end{table*}

\begin{figure}[ht!]
\vspace{-0.1in}
\begin{center}
  \includegraphics[width=1.0\linewidth]{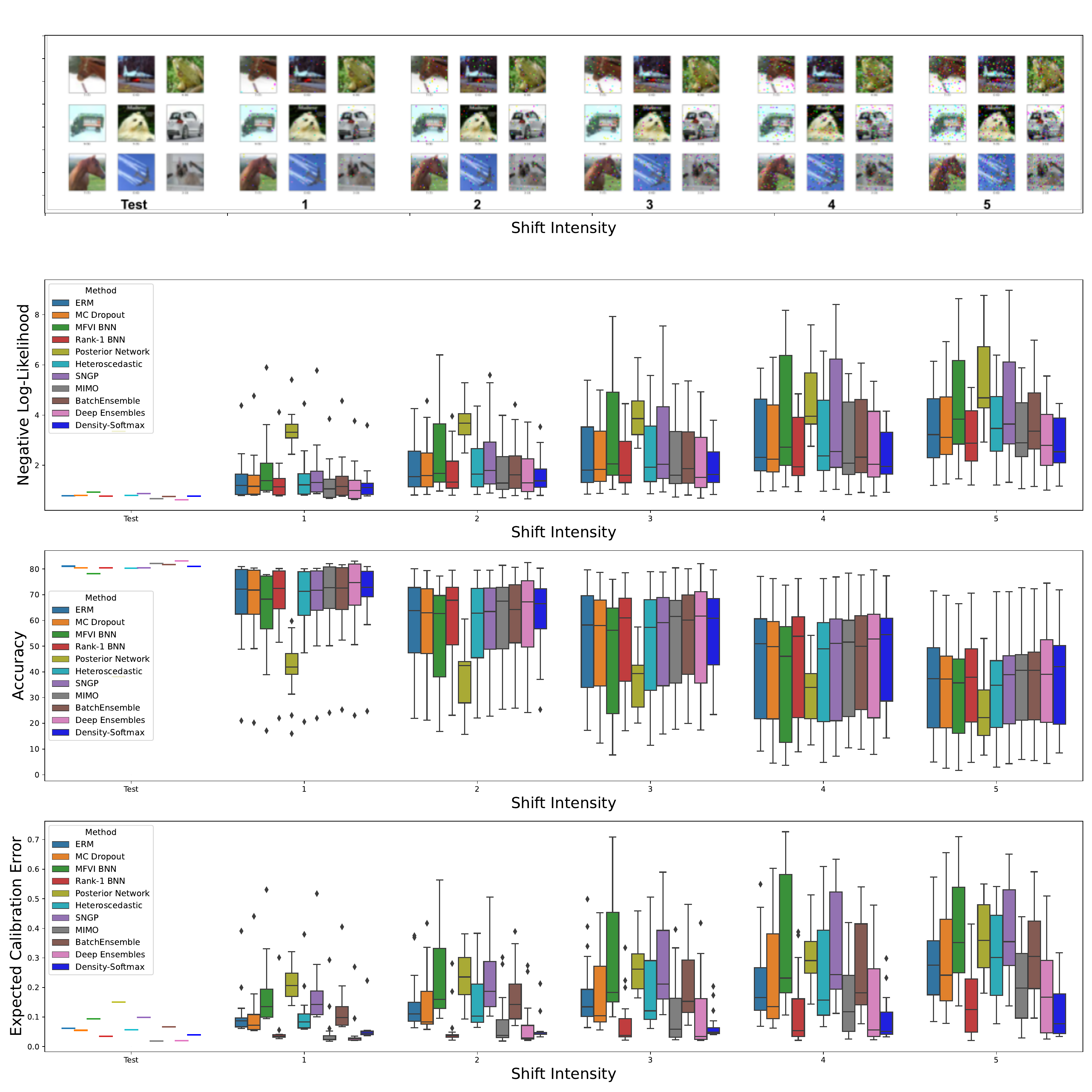}
\end{center}
\vspace{-0.2in}
\caption{A corruption example by shift intensities and comparison under the distributional shift of negative log-likelihood, accuracy, and ECE under all types of corruptions in CIFAR-100-C. For each method, we show the mean on the test set and summarize the results on each shift intensity with a box plot. Each box shows the quartiles summarizing the results across all 17 types of shift while the error bars indicate the min and max across different shift types. Our Density-Softmax achieves the highest accuracy, lowest \acrshort{NLL} and \acrshort{ECE} with a low variance across different shift intensities.} 
\vspace{-0.1in}
\label{fig:cifar100-skew}
\end{figure}
To compare the uncertainty and robustness between methods in more detail, we visualize Figure~\ref{fig:cifar100-skew}, the box plots of the \acrshort{NLL}, accuracy, and \acrshort{ECE} across corruptions under distributional shifts level in CIFAR-100 based dataset from Table~\ref{tab:cifar100_results}. Overall, besides achieving the highest accuracy, lowest \acrshort{NLL} and \acrshort{ECE} on average, our Density-Softmax also always achieves the best accuracy, \acrshort{NLL}, and \acrshort{ECE} performance across different shift intensities in Figure~\ref{fig:cifar100-skew}. For each intensity, our results also show a low variance, implying the stability of our algorithm across different corruption types. More importantly, compared to the two second-best methods Deep Ensembles and Rank-1~\acrshort{BNN}, the gap between our methods increases when the level shift increases, showing our Density-Softmax is more robust than theirs under distributional shifts.

Figure~\ref{fig:cifar100-skew} however only shows the statistics across different shift methods. Therefore, we next analyze in detail the above results over 10 different random seeds. Specifically, we train each model with a specific random seed and then evaluate the result. We repeatedly do this evaluation 10 times with 10 different seed values. Finally, collect 10 of these results together and visualize the mean and standard deviation by the error bar for each shift intensity level. Figure~\ref{fig:10seeds} shows these results in the CIFAR-10-C setting from Table~\ref{tab:cifar10_results}, we observe that the result of our model and other baselines have a small variance, showing the consistency and stability across different random seeds.
\begin{figure}[ht!]
\begin{center}
  \includegraphics[width=1.0\linewidth]{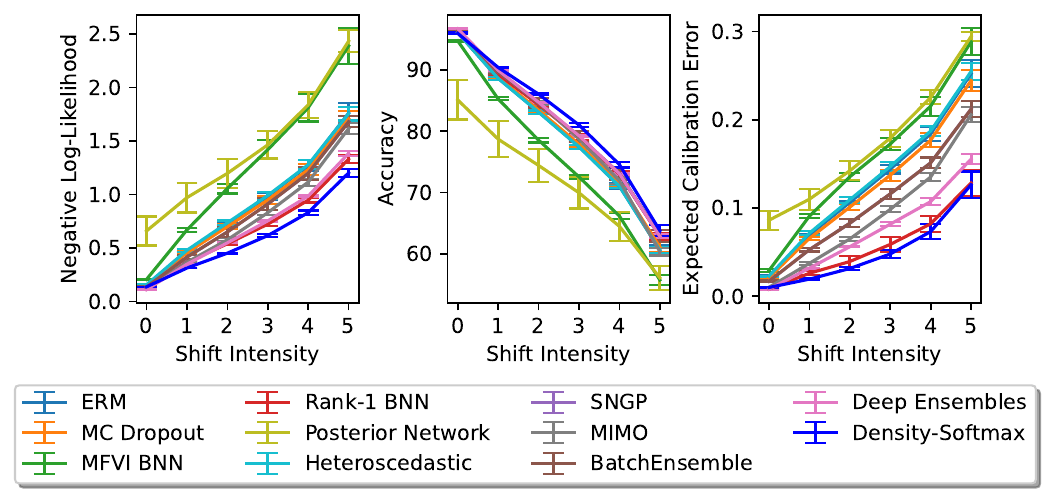}
\end{center}
\vspace{-0.2in}
\caption{Benchmark performance on CIFAR-10-C with Wide~Resnet-28-10 over 10 different random seeds. We plot NLL, accuracy, and ECE for varying corruption intensities; each result is the mean value over 10 runs and 15 corruption types. The error bars represent a fraction of the standard deviation across corruption types. Our method achieves competitive results with \acrshort{SOTA} with low variance across different shift intensities.} 
\label{fig:10seeds}
\end{figure}

Similarly, Figure~\ref{fig:computational_bar} is the comparison across 10 running with 10 different random seeds in terms of model storage and inference cost from Table~\ref{tab:cifar10_results}. There is no variance in storage comparison since the model size is fixed and not affected by a random seed. Meanwhile, the latency variances across different seeds are minor since narrow error bars are in the inference cost comparison bar chart. And this once again, confirms that the results are consistent and stable, our Density-Softmax is the fastest model with almost similar latency to a single Deterministic~\acrshort{ERM}.
\begin{figure}[ht!]
\begin{center}
  \includegraphics[width=1.0\linewidth]{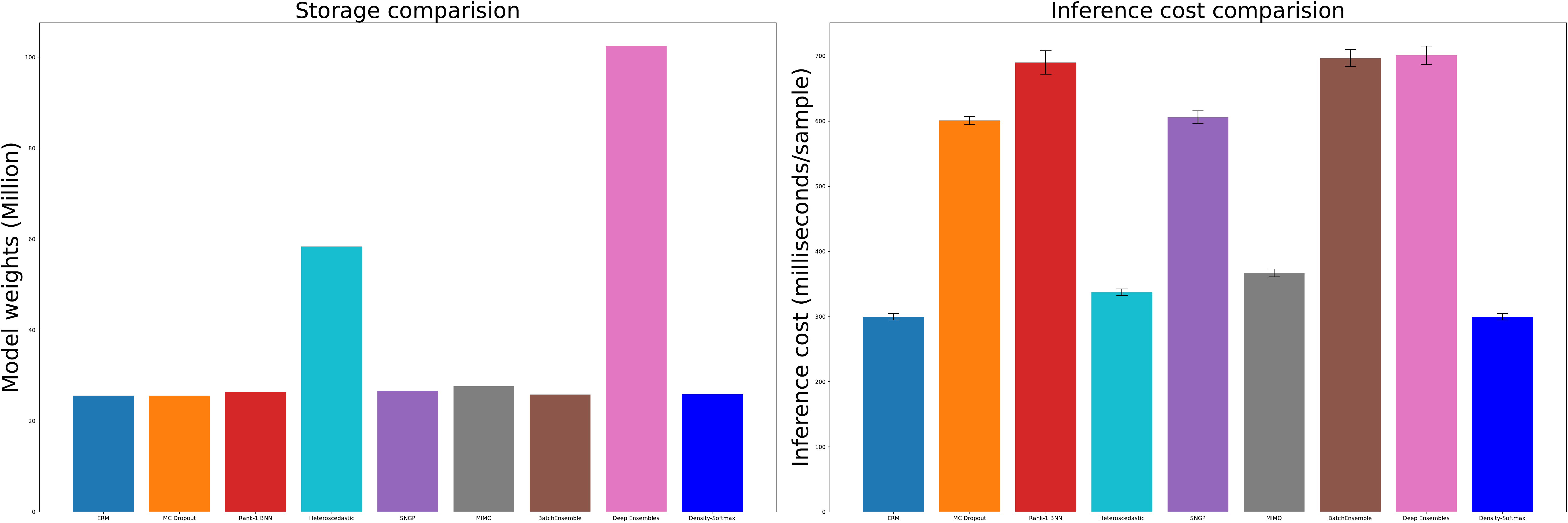}
\end{center}
\vspace{-0.2in}
\caption{Comparison in model storage requirement (Million weights) and inference cost (milliseconds per sample) with error bars across 10 seeds of Resnet-50 on ImageNet dataset with NVIDIA RTX~A5000. Our Density-Softmax achieves almost the same weight and inference speed with a single Deterministic~\acrshort{ERM} model, as a result, outperforming other baselines in computational efficiency.} 
\label{fig:computational_bar}
\end{figure}

\subsubsection{Calibration details}\label{apd:results_calib}
\begin{figure}[ht!]
\begin{center}
  \includegraphics[width=1.0\linewidth]{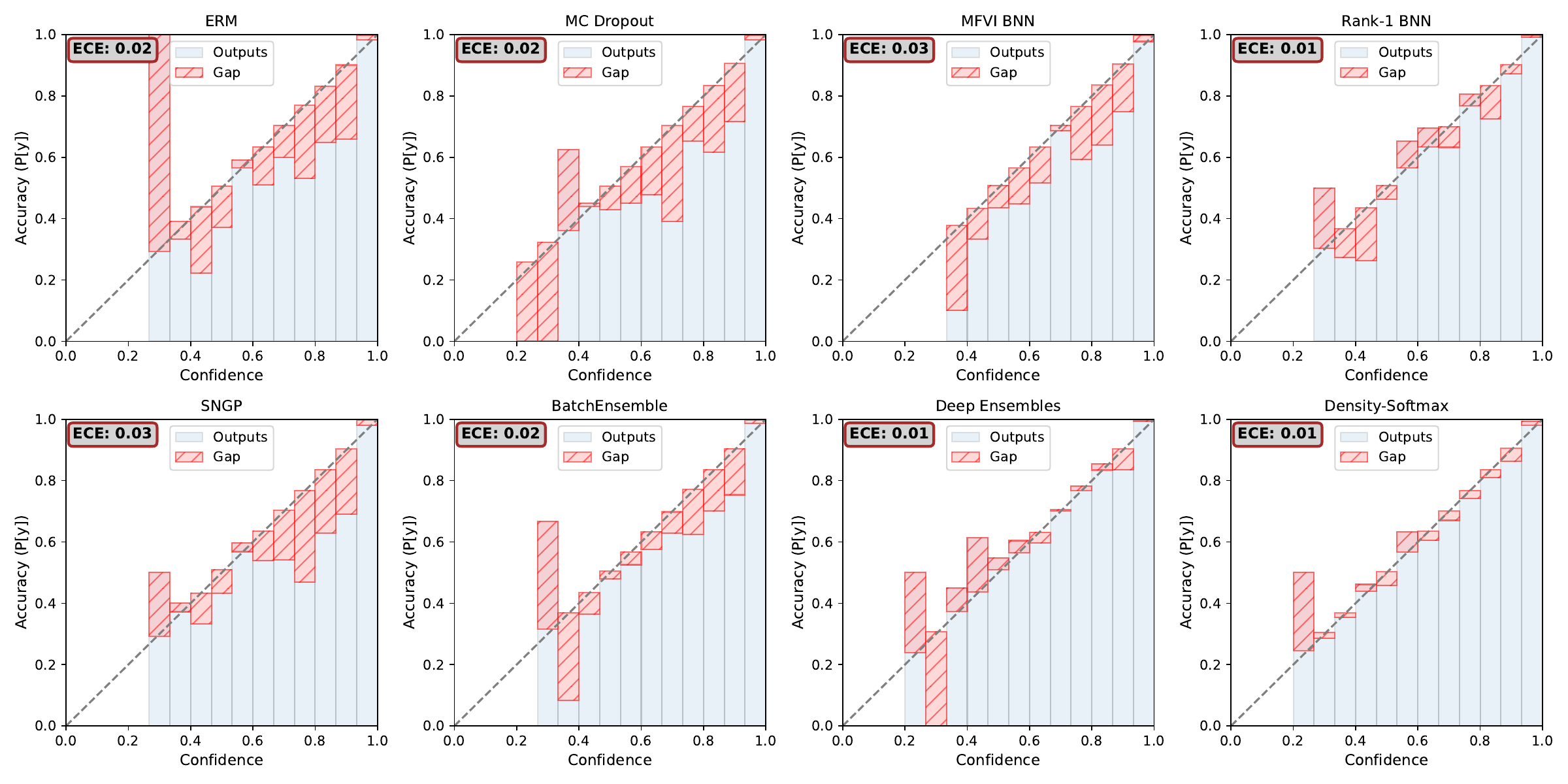}
\end{center}
  \caption{Reliability diagram of Density-Softmax versus different approaches, trained on CIFAR-10, test on \acrshort{IID} CIFAR-10. Density-Softmax has a competitive calibration result with \acrshort{SOTA} methods.}
\label{fig:apd:real_iid_ece_all}
\end{figure}

\begin{figure}[ht!]
\begin{center}
  \includegraphics[width=1.0\linewidth]{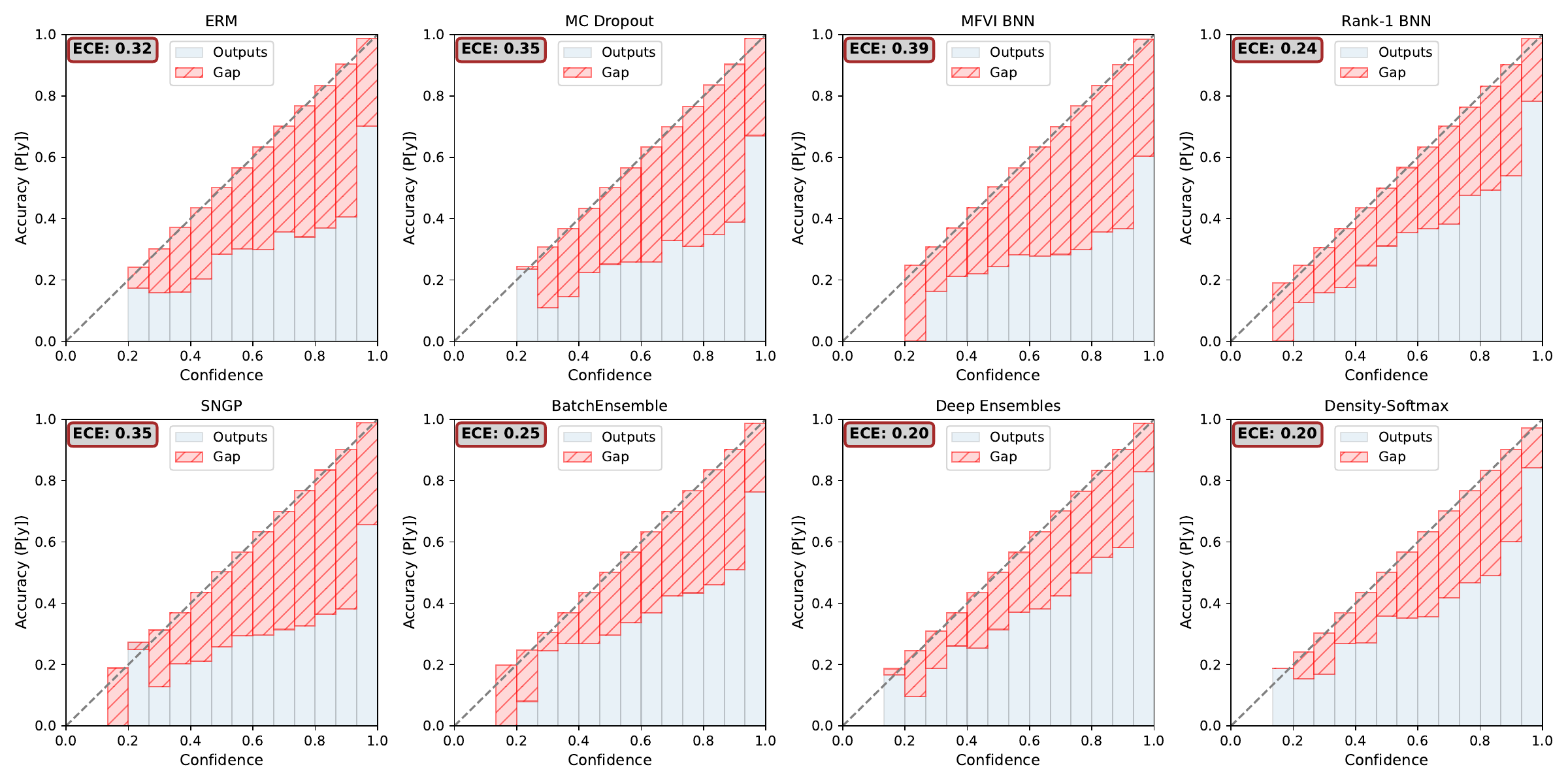}
\end{center}
  \caption{Reliability diagram of Density-Softmax versus different approaches, trained on CIFAR-10, test on \acrshort{OOD} CIFAR-10-C set with "frosted glass blur" skew and "3" intensity. Density-Softmax has a competitive calibration result with \acrshort{SOTA} methods.}
\label{fig:apd:real_out_ece_all}
\end{figure}

To understand how models calibrate in more detail, we visualize the reliability diagrams to test the model's confidence across \acrshort{IID} and \acrshort{OOD} settings. For \acrshort{IID} setting in CIFAR-10, Figure~\ref{fig:apd:real_iid_ece_all} illustrates our Density-Softmax is one of the best-calibrated models with a low \acrshort{ECE}. On the one hand, compared to Deterministic~\acrshort{ERM}, \acrshort{MC}~Dropout, BatchEnsemble, and Deep Ensemble, our model is less under-confidence than them in the prediction that has lower than about $0.4$ confidence. On the other hand, compared to Deterministic~\acrshort{ERM}, \acrshort{MC}~Dropout, \acrshort{MFVI}~\acrshort{BNN}, \acrshort{SNGP}, and BatchEnsemble, our model is less over-confident than them in the prediction that has higher than about $0.4$ confidence. As a result, accompanying Rank-1~\acrshort{BNN} and Deep Ensembles, our model achieves the best calibration performance on \acrshort{IID} data.

More importantly, we find that Density-Softmax is calibrated better than other methods on \acrshort{OOD} data in Figure~\ref{fig:apd:real_out_ece_all} and Figure~\ref{fig:apd:real_ood_ece_all}. In particular, Figure~\ref{fig:apd:real_out_ece_all} shows our model and Deep Ensembles achieves the lowest \acrshort{ECE} and less over-confident than others in a particular CIFAR-10-C test set. Similarly, in real-world distributional shifts CIFAR-10.1 (v6), our model even achieves the lowest \acrshort{ECE} with $0.01$, outperforms the \acrshort{SOTA} Deep Ensembles. These results once again confirm the hypothesis that our model is one of the best reliable models and especially robust under distributional shifts.

\begin{figure}[ht!]
\begin{center}
  \includegraphics[width=1.0\linewidth]{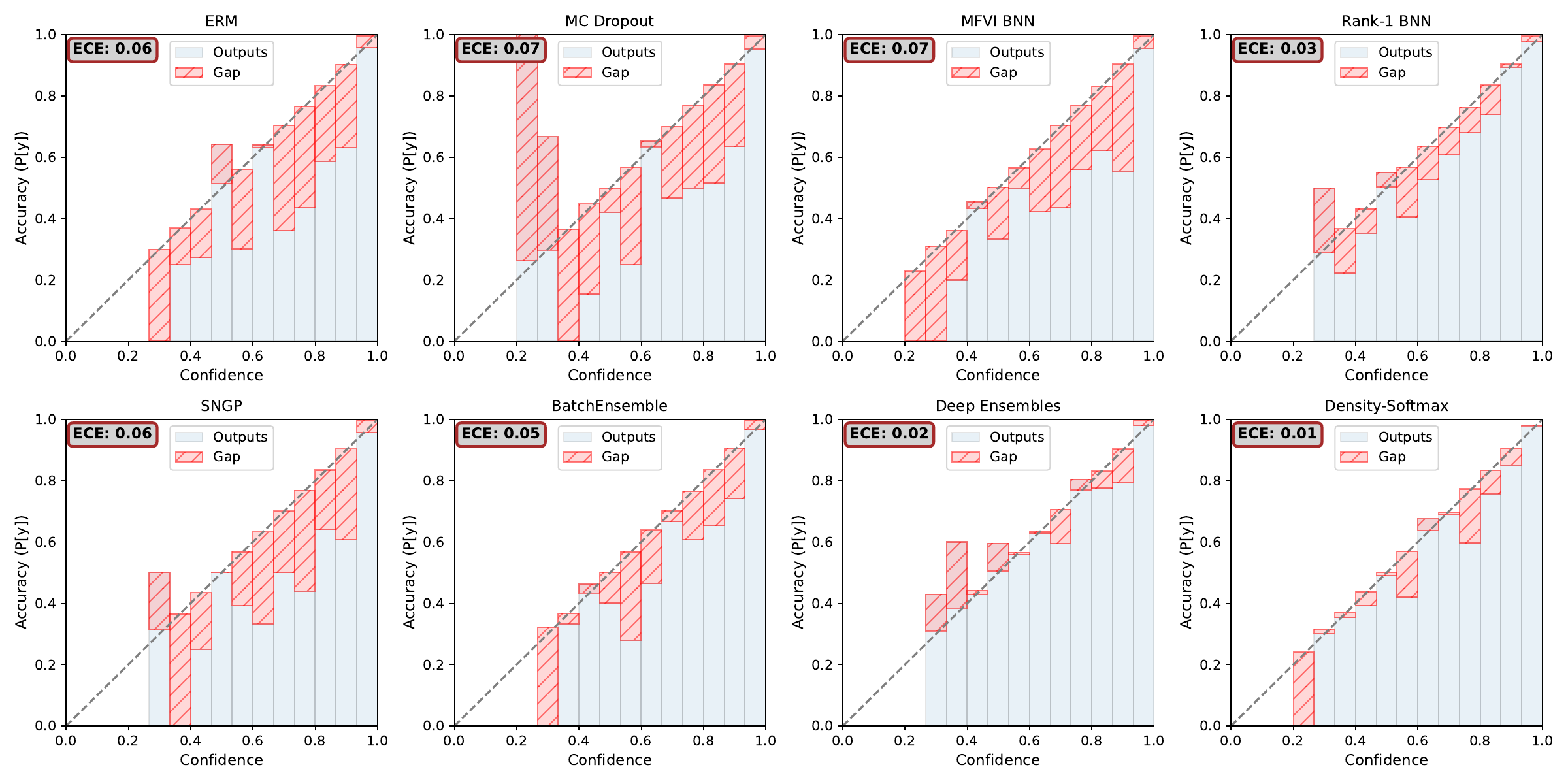}
\end{center}
  \caption{Reliability diagram of Density-Softmax versus different approaches, trained on CIFAR-10, test on real-world shifted \acrshort{OOD} CIFAR-10.1 (v6). Density-Softmax is better-calibrated than other methods.}
\label{fig:apd:real_ood_ece_all}
\end{figure}

\textbf{Miss-classified Expected Calibration Error.} To evaluate the calibration quality in more detail, we continue to make a comparison in terms of Miss-classified \acrshort{ECE}~\citep{chen2022mECE}. This measurement is a specific case of \acrshort{ECE} in Eq.~\ref{eq:ece}. It is different by covering only the miss-classified samples. In this case, the lower m\acrshort{ECE}, the more honest of show uncertainty if the model makes wrong predictions. Fig.~\ref{fig:apd:mece} illustrates the box plots of m\acrshort{ECE} across shift intensities. It confirms that Density-Softmax achieves the best performance with the lowest m\acrshort{ECE}, showing the ability to say "I don't know" before it makes the wrong predictions.

\begin{figure}[t!]
\begin{center}
  \includegraphics[width=1.0\linewidth]{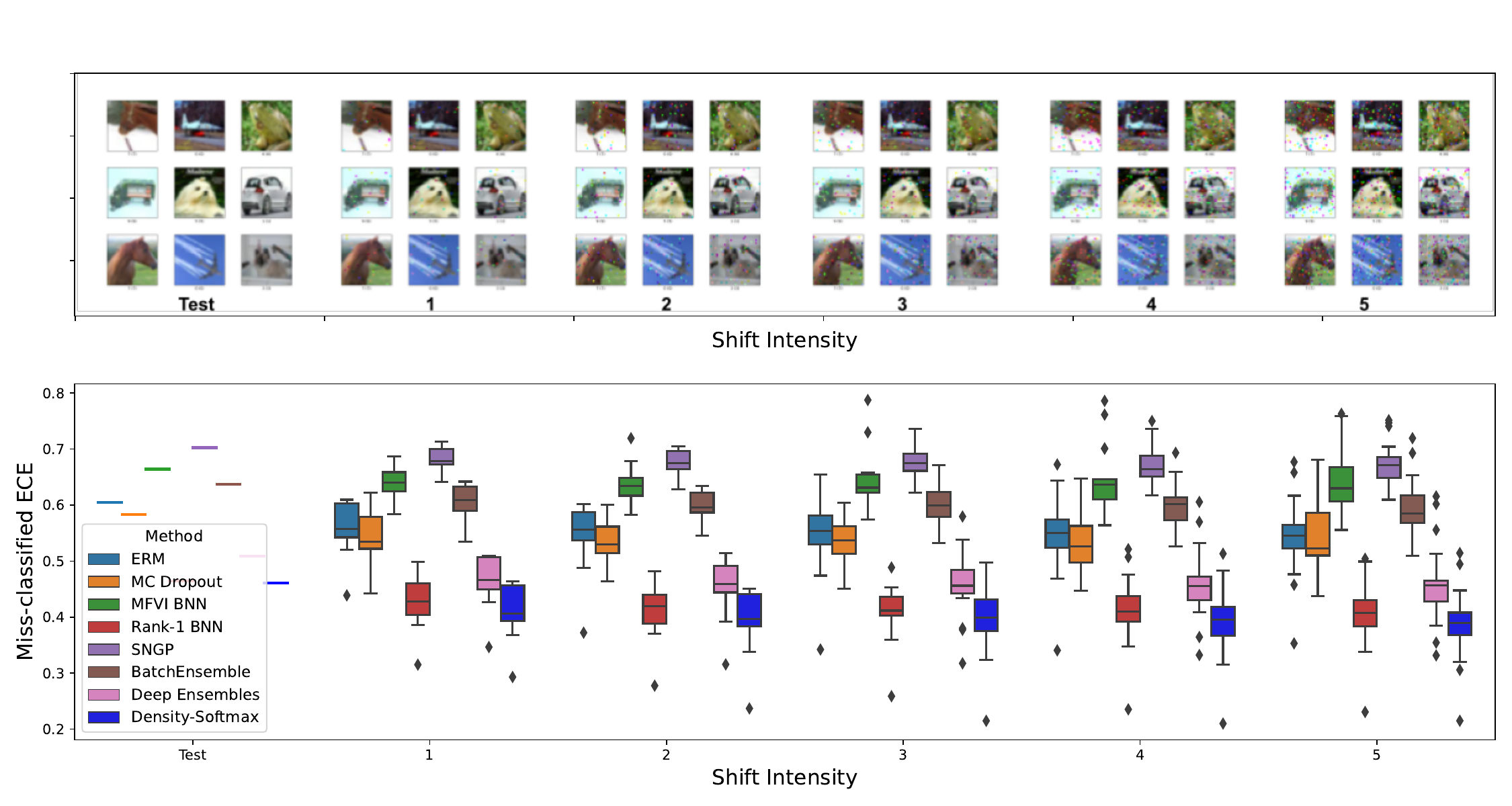}
\end{center}
\caption{A corruption example by shift intensities and comparison under the distributional shift of miss-classified ECE under all types of corruptions in CIFAR-100-C (setting is similar to Figure~\ref{fig:cifar100-skew}). Our Density-Softmax achieves the lowest miss-classified ECE with low variance across different shift intensities.} 
\label{fig:apd:mece}
\end{figure}

\subsubsection{Predictive entropy details}\label{apd:results_entropy}
\begin{figure}[ht!]
\begin{center}
  \includegraphics[width=1.0\linewidth]{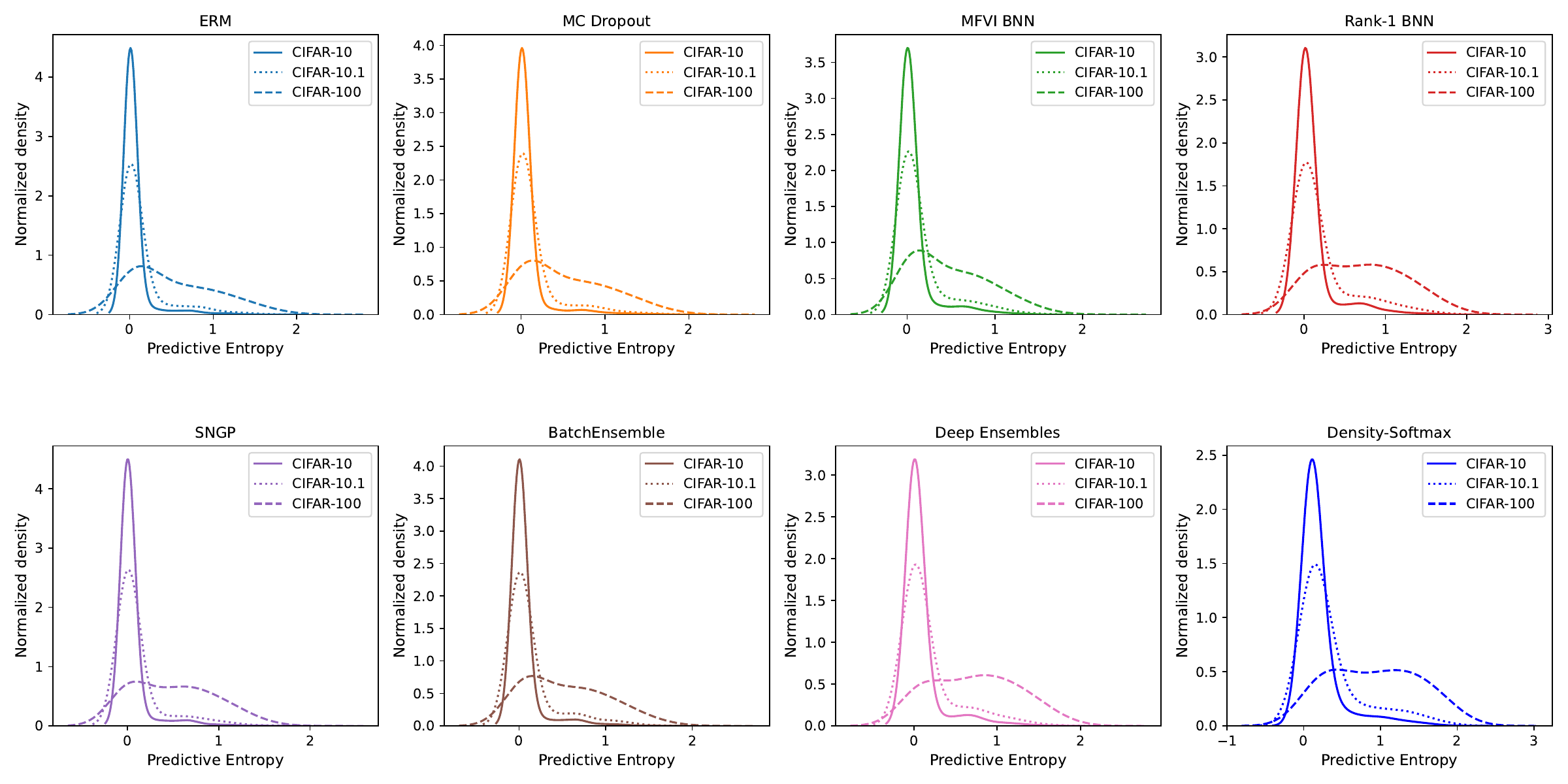}
\end{center}
  \caption{Histogram (bandwidth $=0.5$) with density plot details from Figure~\ref{fig:main}~(a) of predictive entropy for each method on \acrshort{IID} testing data CIFAR-10 (solid lines), covariate shift with \acrshort{OOD} CIFAR-10.1 (v6) (dotted lines), and semantic shift with completely different \acrshort{OOD} CIFAR-100 (dashed lines). Our Density-Softmax has a low predictive entropy value on \acrshort{IID} while achieves the highest entropy value on \acrshort{OOD} data.}
\label{fig:apd:results_entropy}
\end{figure}
Fig.~\ref{fig:main}~(a) has shown our trained model achieves the highest entropy value with a high density. Since this is the semantic shift, the highest entropy value indicates that our model is one of the best \acrshort{OOD} detection models. Yet, this is only a necessary condition for a high-quality uncertainty estimation model because a pessimistic model could also achieve this performance. Therefore, to prove that Density-Softmax is not always under-confidence, we plot Fig.~\ref{fig:apd:results_entropy}. We observe that Density-Softmax has a low predictive entropy with a high density on \acrshort{IID} testing data, proving that our model is not under-confidence on \acrshort{IID} data. Combined with the under-confidence on \acrshort{OOD} data results, these show our Density-Softmax is a reliable model in terms of uncertainty estimation.

\end{document}